\documentclass{article}
\usepackage{times,wrapfig,amsmath,amsfonts,amsthm, amssymb, bm,color,enumitem,algorithm,algpseudocode}
\usepackage{graphicx}
\usepackage{subfigure}
\usepackage{hyperref}
\usepackage{url}
\usepackage{pbox}

\newcommand{\ip}[2]{\left\langle #1, #2 \right\rangle}

\newcommand{\norm}[1]{\left\lVert{#1}\right\rVert}
\newcommand{\abs}[1]{\left\lvert{#1}\right\rvert}

\newcommand{\rank}{\operatorname{rank}}

\newcommand{\trace}{\operatorname{trace}}

\newcommand{\R}{\mathbb{R}}

\newtheorem{thm}{Theorem}[section]
\newtheorem{lem}{Lemma}[section]
\newtheorem*{lem*}{Lemma}
\newtheorem{cor}{Corollary}[section]

\newtheorem{definition}{Definition}[section]

\mathchardef\hyphen="2D

\newcommand{\calA}{\mathcal{A}}

\newcommand{\calN}{\mathcal{N}}

\newcommand{\removed}[1]{}

\newcommand{\eps}{\epsilon}

\renewcommand{\vec}[1]{\bm{#1}}

\newcommand{\mat}[1]{\bm{#1}}

\newcommand\numberthis{\addtocounter{equation}{1}\tag{\theequation}}
\newcommand{\Uo}{\mat{U^*}}
\newcommand{\Xo}{\mat{X^*}}
\newcommand{\Uor}{\mat{U_r^*}}
\newcommand{\Xor}{\mat{X_r^*}}

\newcommand{\A}{\mat{A}}
\newcommand{\X}{\mat{X}}
\newcommand{\Y}{\mat{Y}}

\newcommand{\y}{\vec y}

\newcommand{\U}{\mat{U}}
\newcommand{\V}{\mat{V}}
\newcommand{\ZZ}{\mat{Z}}

\newcommand{\B}{\mat{B}}

\newcommand{\remove}[1]{}
\usepackage[numbers, compress]{natbib}
\usepackage{fullpage}

\usepackage[utf8]{inputenc} 
\usepackage[T1]{fontenc}    
\usepackage{hyperref}       
\usepackage{url}            
\usepackage{booktabs}       
\usepackage{amsfonts}       
\usepackage{nicefrac}       
\usepackage{microtype}      
\usepackage{authblk}
\title{Global Optimality of Local Search \\ for Low Rank Matrix Recovery}

\author{Srinadh Bhojanapalli\thanks{srinadh@ttic.edu}, Behnam Neyshabur\thanks{bneyshabur@ttic.edu}, Nathan Srebro\thanks{nati@ttic.edu}}
\affil{Toyota Technological Institute at Chicago}

\date{}

\begin{document}

\maketitle

\begin{abstract}
  We show that there are no spurious local minima in the non-convex
  factorized parametrization of low-rank matrix recovery from
  incoherent linear measurements.  With noisy measurements we show all
  local minima are very close to a global optimum.  Together with
  a curvature bound at saddle points, this yields a polynomial time
  global convergence guarantee for stochastic gradient descent {\em
    from random initialization}.
\end{abstract}

\section{Introduction}\label{sec:intro}
Low rank matrix recovery problem is heavily studied and has numerous applications in collaborative filtering, quantum state tomography, clustering, community detection, metric learning and multi-task learning \cite{rennie2005fast, gross2010quantum, flammia2012quantum, yu2014large}.

We consider the ``matrix sensing'' problem of recovering a low-rank
(or approximately low rank) p.s.d.~matrix\footnote{We study the case
  where $\Xo$ is PSD.  We believe the techniques developed here can be
  used to extend results to the general case.} $\Xo\in\R^{n\times n}$,
given a linear measurement operator $\calA:\R^{n \times n} \to \R^m$
and noisy measurements $\vec{y} =\calA(\Xo) +\vec{w}$, where $\vec{w}$ is an
i.i.d.~noise vector.  An estimator for $\Xo$ is given by the
rank-constrained, non-convex problem
\begin{equation}\label{eq:prob1} \underset{\X:\rank(\X) \leq r
  }{\text{minimize}} ~~  \|\calA(\X) - \vec{y}\|^2. \end{equation} 
This matrix sensing problem has received considerable attention
recently \cite{zheng2015convergent, zhao2015nonconvex, tu2015low}.  This and other rank-constrained problems are common in machine learning and related fields, and have been used for applications discussed above. A typical
theoretical approach to low-rank problems, including \eqref{eq:prob1}
is to relax the low-rank constraint to a convex constraint, such as the trace-norm of $\X$.  Indeed, for matrix sensing,
\citet{recht2010guaranteed} showed that if the measurements are
noiseless and the measurement operator $\calA$ satisfies a restricted
isometry property, then a low-rank $\Xo$ can be recovered as the
unique solution to a convex relaxation of \eqref{eq:prob1}.
Subsequent work established similar guarantees also for the noisy
and approximate case \cite{jain2010guaranteed, candes2011tight}.

However, convex relaxations to the rank are not the common approach
employed in practice.  In this and other low-rank problems, the method
of choice is typically unconstrained local optimization (via
e.g.~gradient descent, SGD or alternating minimization) on the
factorized parametrization
\begin{equation}\label{eq:prob2}
  \underset{\U \in \R^{n \times r}}{\text{minimize}}~ f(\U)= \|\calA(\U\U^\top)-\vec{y}\| ^2,
\end{equation}
where the rank constraint is enforced by limiting the dimensionality
of $\U$.  Problem \eqref{eq:prob2} is a non-convex optimization
problem that could have many bad local minima (as we show in Section
\ref{sec:necessity}), as well as saddle points.  Nevertheless, local optimization
seems to work very well in practice.  Working on \eqref{eq:prob2} is
much cheaper computationally and allows scaling to large-sized
problems---the number of optimization variables is only $O(nr)$ rather
than $O(n^2)$, and the updates are usually very cheap, especially
compared to typical methods for solving the SDP resulting from the
convex relaxation. There is therefore a significant disconnect between
the theoretically studied and analyzed methods (based on convex
relaxations) and the methods actually used in practice.

Recent attempts at bridging this gap showed that, some form of global
``initialization'', typically relying on singular value decomposition,
yields a solution that is already close enough to $\Xo$; that local
optimization from this initializer gets to the global optima (or to a
good enough solution). \citet{jain2013low, keshavan2012efficient}
proved convergence for alternating minimization algorithm provided the
starting point is close to the optimum, while
\citet{zheng2015convergent, zhao2015nonconvex, tu2015low,
  chen2015fast,bhojanapalli2015dropping} considered gradient
descent methods on the factor space and proved local convergence.  But
all these studies rely on global initialization followed by local
convergence, and do not tackle the question of the existence of
spurious local minima or deal with optimization starting from random
initialization.  There is therefore still a disconnect between this
theory and the empirical practice of starting from random
initialization and relying {\em only} on the local search to find the
global optimum.

In this paper we show that, under a suitable incoherence condition on
the measurement operator $\calA$ (defined in Section 2), with
noiseless measurements and with $\rank(\Xo)\leq r$, the problem
\eqref{eq:prob2} has no spurious local minima (i.e.~all local minima
are global and satisfy $\Xo=\U\U^\top$).  Furthermore, under the same
conditions, all saddle points have a direction with significant
negative curvature, and so using a recent result of
\citet{ge2015escaping} we can establish that stochastic gradient
descent from random initialization converges to $\Xo$ in
polynomial number of iterations.  We extend the results also to the noisy
and approximately-low-rank settings, where we can guarantee
that every local minima is close to a global minimum.  The incoherence
condition we require is weaker than conditions used to establish
recovery through local search, and so our results also ensures
recovery in polynomial time under milder conditions than what was
previously known.  In particular, with i.i.d.~Gaussian measurements, we ensure
no spurious local minima and recovery through local search with the
optimal number $O(nr)$ of measurements.\vspace{-0.1in}

\paragraph{Related Work} Our work is heavily inspired by
\citet{bandeira2016low}, who recently showed similar behavior for the
problem of community detection---this corresponds to a specific
rank-$1$ problem with a linear objective, elliptope constraints and a binary solution.  Here
we take their ideas, extend them and apply them to matrix sensing with
general rank-$r$ matrices.  In the past several months, similar type
of results were also obtained for other non-convex problems (where the
source of non-convexity is {\em not} a rank constraint), specifically
complete dictionary learning~\cite{sun2015complete} and phase
recovery~\cite{sun2016geometric}.  A related recent result of a
somewhat different nature pertains to rank unconstrained linear
optimization on the elliptope, showing that local minima of the
rank-constrained problem approximate well the global optimum of the
rank {\rm unconstrained} convex problem, even though they might {\em
  not} be the global minima (in fact, the approximation guarantee for
the actual global optimum is better) \cite{montanari2016grothendieck}.

Another non-convex low-rank problem long known to not possess spurious
local minima is the PCA problem, which can also be phrased as matrix
approximation with full observations, namely $\min_{\rank(\X)\leq
  r}\norm{A-X}_F$ (e.g. \cite{srebro2003weighted}). Indeed, local search methods such as the power-method are routinely used for this problem.  Recently local
optimization methods for the PCA problem working more directly on the
optimized formulation have also been studied, including SGD
\cite{sa2015global} and Grassmannian optimization \cite{zhang2015global}.  These
results are somewhat orthogonal to ours, as they study a setting in
which it is well known there are never any spurious local minima, and
the challenge is obtaining satisfying convergence rates.

The seminal work of \citet{burer2003nonlinear} proposed low-rank factorized
optimization for SDPs, and showed that for extremely high rank
$r>\sqrt{m}$ (number of constraints), an Augmented Lagrangian method converges asymptotically to the optimum.  It was also shown that (under mild
conditions) any rank deficient local minima is a global minima
\cite{burer2005local, journee2010low}, providing a
post-hoc verifiable sufficient condition for global optimality.
However, this does not establish any a-priori condition, based on
problem structure, implying the lack of spurious local minima.

While preparing this manuscript, we also became aware of parallel work
\cite{ge2016matrix} studying the same question for the related but different problem of matrix completion. For this problem they obtain a similar guarantee, though with suboptimal dependence on the incoherence parameters and so suboptimal sample
complexity, and requiring adding a specific non-standard regularizer to
the objective---this is not needed for our matrix sensing results.

We believe our work, together with the parallel work of \cite{ge2016matrix}, are
the first to establish the lack of spurious local minima and the
global convergence of local search from random initialization for a
non-trivial rank-constrained problem (beyond PCA with
full observations) with rank $r>1$.

\noindent {\bf Notation.}  For matrices $\X, \Y \in \R^{n \times n}$,
their inner product is $\ip{\X}{\Y} = \trace\left(\X^\top \Y \right)$.
We use $\|\X\|_F$, $\| \X\|_2$ and $\| \X \|_{*}$ for the Frobenius,
spectral and nuclear norms of a matrix respectively. Given a matrix
$\X$, we use $\sigma_i\left(\X\right)$ to denote singular values of
$\X$ in decreasing order. $\X_r = \arg\min_{\rank(\Y)\leq r}\norm{\X-\Y}_F$ denotes the rank-$r$ approximation of $\X$,
as obtained via its truncated singular value decomposition. We use
plain capitals $R$ and $Q$ to denote orthonormal
matrices. 

\section{Formulation and Assumptions}\label{sec:prelim}
We write the linear measurement operator $\calA:\R^{n \times n} \to
\R^m$ as $\calA(\X)_i = \ip{\A_i}{\X}$ where $\A_i\in\R^{n\times n}$,
yielding $y_i =\ip{\A_i}{\Xo } +w_i, i=1,\cdots,m$.  We assume $w_i
\sim \mathcal{N}(0, \sigma_w^2)$ is i.i.d Gaussian noise.  We are
generally interested in the high dimensional regime where the number
of measurements $m$ is usually much smaller than the dimension $n^2$.

Even if we know that $\rank(\Xo)\leq r$, having many measurements
might not be sufficient for recovery if they are not ``spread out''
enough.  E.g., if all measurements only involve the first $n/2$ rows
and columns, we would never have any information on the bottom-right
block. A sufficient condition for identifiability of a low-rank $\Xo$
from linear measurements by \citet{recht2010guaranteed} is based on restricted isometry property defined below.
\begin{definition}[Restricted Isometry Property]
Measurement operator $\calA: \R^{n \times n} \to \R^m$ (with rows $\A_i$, $i=1,\cdots, m$) satisfies $(r, \delta_r)$ RIP if for any $n \times n$ matrix $\X$ with rank $\leq r$, \begin{equation}\label{eq:rip}(1-\delta_r) \|\X\|_F^2 \leq \frac{1}{m}\sum_{i=1}^m \ip{\A_i}{\X}^2 \leq (1+\delta_r)\|\X\|_F^2.\end{equation}
\end{definition}

In particular, $\Xo$ of rank $r$ is identifiable if $\delta_{2r}<1$
\citep[see][Theorem 3.2]{recht2010guaranteed}.  One situation in which RIP is obtained is for random measurement operators.  For example, matrices with
i.i.d.~$\mathcal{N}(0, 1)$ entries satisfy $(r,\delta_r)$-RIP when $m
=O(\frac{nr}{\delta^2})$~\citep[see][Theorem~2.3]{candes2011tight}.
This implies identifiability based on i.i.d.~Gaussian measurement
with $m=O(nr)$ measurements (coincidentally, the number of degrees of
freedom in $\Xo$, optimal up to a constant factor).


\section{Main Results}\label{sec:main}

We are now ready to present our main result about local minima for the
matrix sensing problem~\eqref{eq:prob2}.  We first present the results
for noisy sensing of exact low rank matrices, and then generalize the
results also to approximately low rank matrices.  

Now we will present our result characterizing local minima of $f(\U)$, for low-rank $\Xo$. Recall that measurements are $\vec{y} =\calA(\Xo) + \vec{w}$, where entries of $\vec{w}$ are i.i.d. Gaussian - $w_i \sim \mathcal{N}(0,\sigma_w^2)$.

\begin{thm}\label{thm:exact_noise}
  Consider the optimization problem \eqref{eq:prob2} where
  $\vec{y}=\calA(\Xo)+\vec{w}$, $\vec{w}$ is i.i.d.~$\calN(0,\sigma_w^2),$
  $\calA$ satisfies $(2r,\delta_{2r})$-RIP with $\delta_{2r}<\frac{1}{10}$, and
  $\rank(\Xo)\leq r$. Then, with probability $\geq 1 -\frac{10}{n^2}$
  (over the noise), for any local minimum $\U$ of $f(\U)$: 
  $$\|\U\U^\top-\Xo \|_F \leq 20 \sqrt{\frac{\log(n)}{m}}\sigma_w.$$
\end{thm}
In particular, in the noiseless case ($\sigma_w=0$) we have
$\U\U^\top=\Xo$ and so $f(\U)=0$ and every local minima is global.  In
the noiseless case, we can also relax the RIP requirement to
$\delta_{2r}<1/5$ (see Theorem \ref{thm:exact} in Section \ref{sec:proof_exact}).  In the
noisy case we cannot expect to ensure we always get to an exact global
minima, since the noise might cause tiny fluctuations very close to
the global minima possibly creating multiple very close local minima.
But we show that all local minima are indeed very close to some
factorization $\Uo{\Uo}^\top=\Xo$ of the true signal, and hence to a global
optimum, and this ``radius'' of local minima decreases as we have
more observations.

The proof of the Theorem for the noiseless case is presented in
Section \ref{sec:proof_exact}.  The proof for the general setting
follows along the same lines and can be found in the Appendix.


So far we have discussed how all local minima are global, or at least very
close to a global minimum.  Using a recent result by
\citet{ge2015escaping} on the convergence of SGD for non-convex
functions, we can further obtain a polynomial bound on the number of
SGD iterations required to reach the global minima.  The main
condition that needs to be established in order to ensure this, is that
all saddle points of \eqref{eq:prob2} satisfy the ``strict saddle
point condition'', i.e.~have a direction with significant negative
curvature:
\begin{thm}[Strict saddle]\label{thm:strict_saddle}
  Consider the optimization problem \eqref{eq:prob2} in the noiseless
  case, where $\vec{y}=\calA(\Xo)$, $\calA$ satisfies $(2r,\delta_{2r})$-RIP
  with $\delta_{2r}<\frac{1}{10}$, and $\rank(\Xo)\leq r$.  Let $\U$
  be a first order critical point of $f(\U)$ with $\U \U^\top \neq
  \Xo$.  Then the smallest eigenvalue of the Hessian satisfies
\begin{equation*}
\lambda_{\min}\left[\frac{1}{m}\nabla^2 (f(\U))\right] \leq  \frac{-4}{5} \sigma_r(\Xo).
\end{equation*}
\end{thm} 

Now consider the stochastic gradient descent updates, 
\begin{equation}\label{eq:grad_updates}
\U^+ = \textrm{Proj}_b\left(\U -\eta \left( \sum_{i=1}^m ( \ip{\A_i}{\U\U^\top} -y_i ) \A_i \U + \psi \right)\right),
\end{equation} 
where $\psi$ is uniformly distributed on the unit sphere and
$\textrm{Proj}_b$ is a projection onto $\norm{\U}_F \leq b$.  Using
Theorem \ref{thm:strict_saddle} and the result of
\citet{ge2015escaping} we can establish:

\begin{thm}[Convergence from random initialization]
  \label{thm:grad_convergence}
  Consider the optimization problem \eqref{eq:prob2} under the same
  noiseless conditions as in Theorem \ref{thm:strict_saddle}.  Using
  $b\geq \norm{\U^*}_F$, for some global optimum $\U^*$ of $f(\U)$, for
  any $\epsilon,c>0$, after $T=poly\left(\frac{1}{\sigma_r(\Xo)},
    \sigma_1(\Xo), b, \frac{1}{\eps}, \log(1/c)\right)$ iterations of
  \eqref{eq:grad_updates} with an appropriate stepsize $\eta$,
  starting from a random point uniformly distributed on $\norm{\U}_F=b$,
  with probability at least $1-c$, we reach an iterate $\U_T$ satisfying
$$\|\U_T -\Uo \|_F \leq \eps.$$
\end{thm}

The above result guarantees convergence of noisy gradient descent to a
global optimum.  Alternatively, second order methods such as cubic
regularization~(\citet{nesterov2006cubic}) and trust
region~(\citet{cartis2012complexity}) that have guarantees based on the
strict saddle point property can also be used here.


{\bf RIP Requirement:} Our results require $(2r,1/10)$-RIP for the noisy case and $(2r,1/5)$-RIP for the noiseless case.  Requiring
$(2r,\delta_{2r})$-RIP with $\delta_{2r}<1$ is sufficient to ensure
uniqueness of the global optimum of \eqref{eq:prob1}, and thus
recovery in the noiseless setting \cite{recht2010guaranteed}, but all
known efficient recovery methods require stricter conditions.  The
best guarantees we are aware of require $(5r,1/10)$-RIP
\cite{recht2010guaranteed} or $(4r,0.414)$-RIP \cite{candes2011tight} using a convex
relaxation.  Our requirement is not directly comparable to the latter,
as we require RIP on a smaller set, but with a lower (stricter)
$\delta$.  Alternatively, $(6r,1/10)$-RIP is required for global
initialization followed by non-convex
optimization~\cite{tu2015low}---our requirement is strictly better.
In terms of requirements on $(2r,\delta_{2r})$-RIP for non-convex methods, the best we are aware of is requiring $\delta_{2r}<\Omega(1/r)$~\cite{jain2013low,
  zhao2015nonconvex, zheng2015convergent}--this is a much stronger
condition than ours, and it yields a suboptimal required number of
spherical Gaussian measurements of $\Omega(nr^3)$. So, compared
to prior work our requirement is very mild---it ensures efficient recovery even
in a regime not previously covered by any guarantee on efficient
recovery, and requires the optimal number of spherical Gaussian
measurements (up to a constant factor) of $O(nr)$.

\paragraph{Extension to Approximate Low Rank}\label{sec:high_rank}

We can also obtain similar results that deteriorate gracefully if
$\Xo$ is not exactly low rank, but is close to being low-rank (see
proof in the Appendix):
\begin{thm}\label{thm:inexact_sym}
  Consider the optimization problem \eqref{eq:prob2} where
  $\vec{y}=\calA(\Xo)$ and $\calA$ satisfies $(2r,\delta_{2r})$-RIP with
  $\delta_{2r}<\frac{1}{100}$, Then, for any local
  minima $\U$ of $f(\U)$:
  $$\|\U\U^\top-\Xo \|_F \leq 4 ( \|\Xo-\Xor\|_F +  \delta_{2r} \|\Xo
  -\Xor\|_{*}),$$ where $\Xor$ is the best rank $r$ approximation of
  $\Xo$.
\end{thm}

This theorem guarantees that any local optimum of $f(\U)$ is close to
$\Xo$ upto an error depending on $\|\Xo-\Xor\|$.  For the low-rank
noiseless case we have $\Xo=\Xor$ and the right hand side vanishes.
When $\Xo$ is not exactly low rank, the best recovery error we can
hope for is $\norm{\Xo-\Xor}_F$, since $\U\U^\top$ is at most rank $k$.
 On the right hand side of Theorem \ref{thm:inexact_sym}, we have also
 a nuclear norm term, which might be higher, but it also gets scaled
 down by $\delta_{2r}$, and so by the number of measurements.
 
\begin{figure}[ht]
	\begin{center}
		\includegraphics[width=0.3\textwidth]{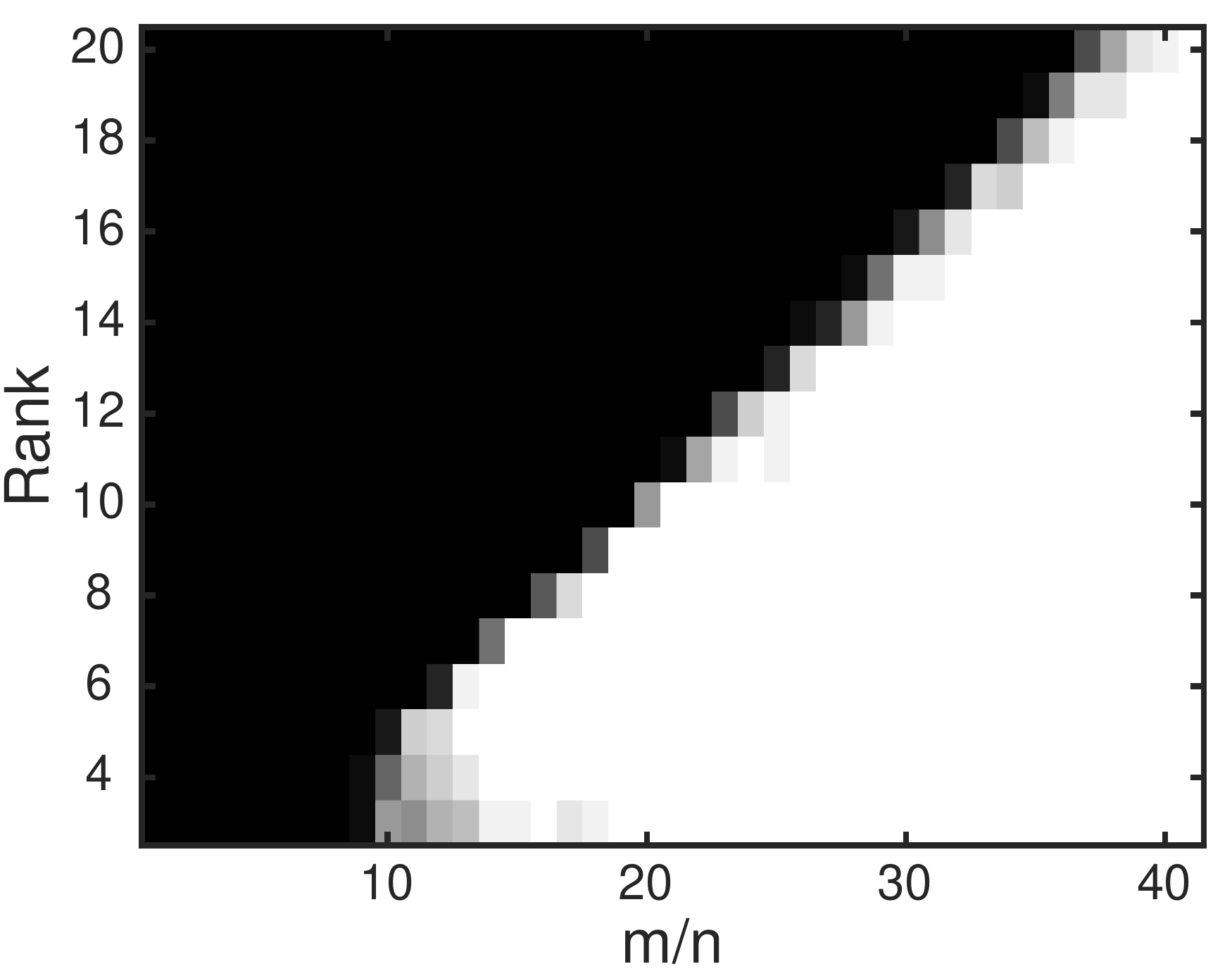} 
		\hspace{0.05in}
		\includegraphics[width=0.3\textwidth]{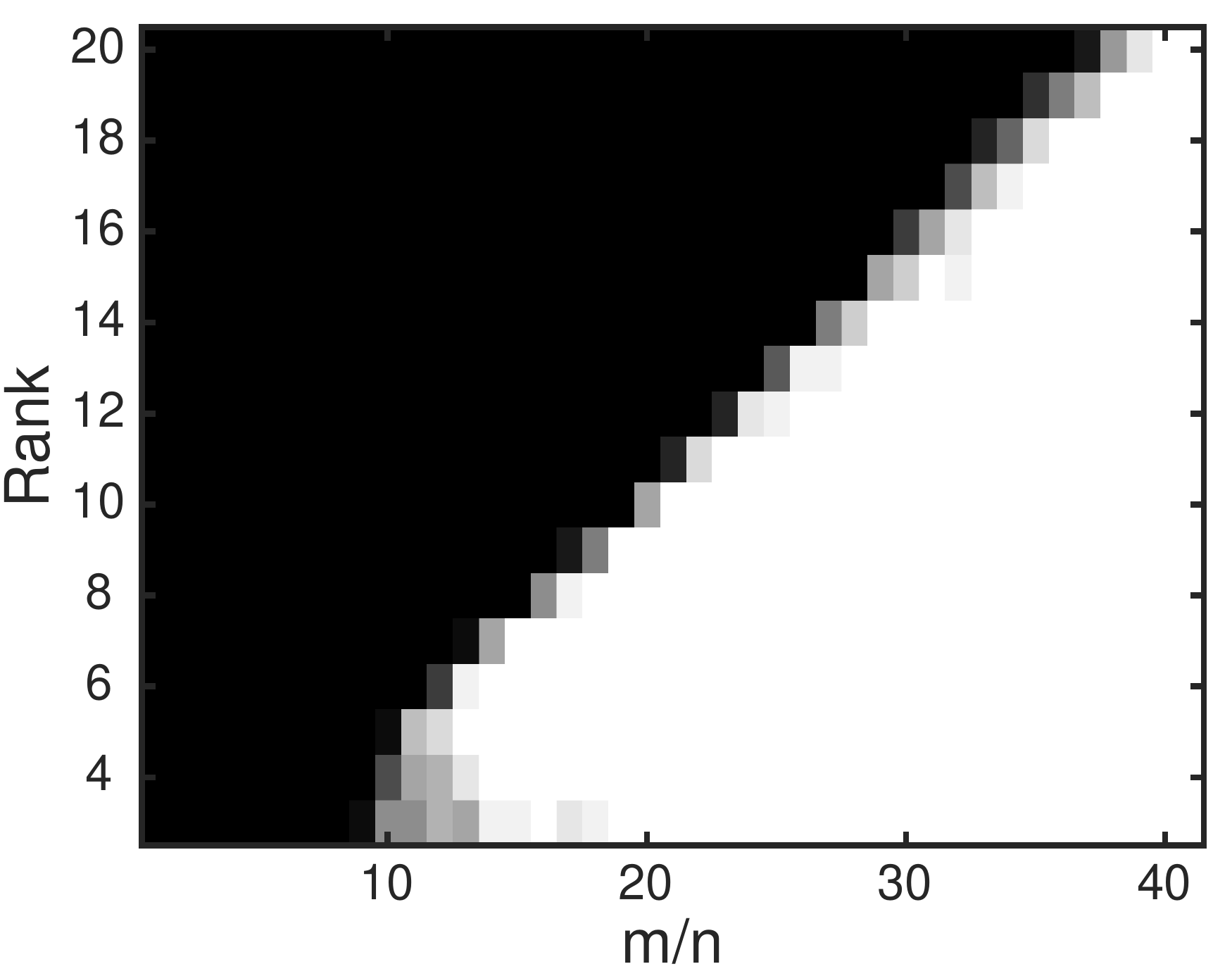}
		\hspace{0.05in}
		\includegraphics[width=0.3\textwidth]{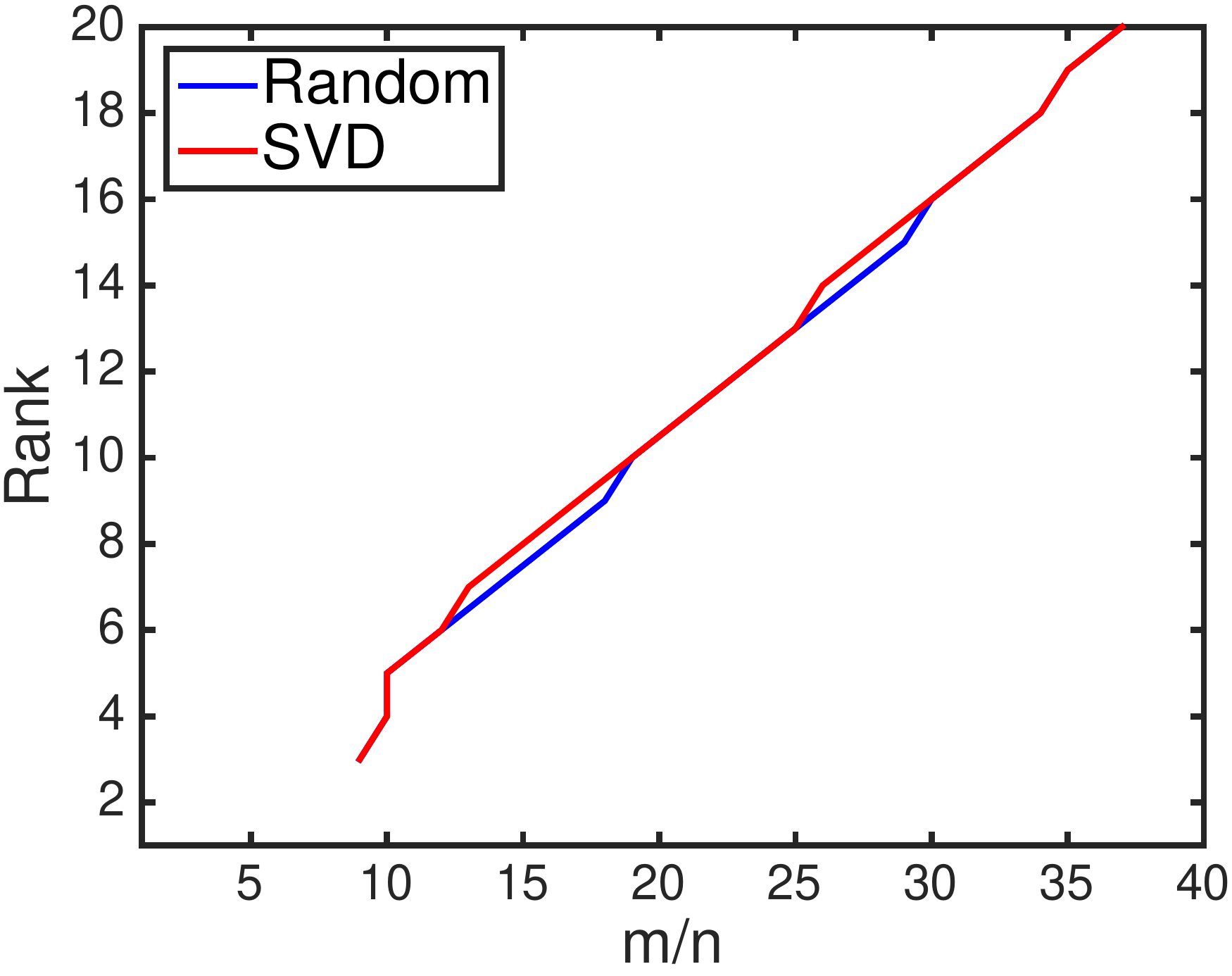}
	\end{center}
	\caption{The plots in this figure compare the success probability of gradient descent between (left) random and (center) SVD initialization (suggested in~\cite{jain2013low}), for problem~\eqref{eq:prob2}, with increasing number of samples $m$ and various values of rank $r$. Right most plot is the first $m$ for a given $r$, where the probability of success reaches the value $0.5$. A run is considered success if $\|\U \U^\top -\Xo\|_F/\|\Xo\|_F \leq 1e-2$. White cells denote success and black cells denote failure of recovery. We set $n$ to be $100$. Measurements $y_i$ are inner product of entrywise i.i.d Gaussian matrix and a rank-$r$ p.s.d matrix with random subspace. We notice no significant difference between the two initialization methods, suggesting absence of local minima as shown. Both methods have phase transition around $m =2\cdot n\cdot r$.}
{\label{fig2_main}}
\end{figure}
\vspace{-0.1in}
\section{Proof for the Noiseless Case}\label{sec:proof_exact}
In this section we present the proof characterizing the local minima of problem~\eqref{eq:prob2}. For ease of exposition we first present the results for the noiseless case ($\vec{w}=0$). Proof for the general case can be found in the Appendix.

\begin{thm}\label{thm:exact}
Consider the optimization problem \eqref{eq:prob2} where
  $\vec{y}=\calA(\Xo)$, $\calA$ satisfies $(2r,\delta_{2r})$-RIP with $\delta_{2r}<\frac{1}{5}$, and  $\rank(\Xo)\leq r$. Then, for any local minimum $\U$ of $f(\U)$:  $$\U\U^\top=\Xo.$$
\end{thm}

For the proof of this theorem we first discuss the implications of the first and second order optimality conditions and then show how to combine them to yield the result.

Our proof techniques are different from existing results characterizing local minima of dictionary learning~\cite{sun2015complete}, phase retrieval~\cite{sun2016geometric} and community detection~\cite{bandeira2016low, montanari2016grothendieck}. For most of these problems Hessian is PSD only for points close to optimum. However, Hessian of $f(\U)$ can be PSD even for points far from to optima. Hence we need new directions to use the second order conditions. 

Invariance of $f(\U)$ over $r \times r$ orthonormal matrices introduces additional challenges in comparing a given stationary point to a global optimum. We have to find the best orthonormal matrix $R$ to align a given stationary point $\U$ to a global optimum $\Uo$, where $\Uo {\Uo}^\top =\Xo$, to combine results from the first and second order conditions, without degrading the isometry constants.

Consider a local optimum $\U$ that satisfies first and second order optimality conditions of problem~\eqref{eq:prob2}. In particular $\U$ satisfies $\nabla f(\U) =0 $ and $z^\top \nabla^2 f(\U) z \geq 0$ for any $z \in \R^{n\cdot r}$. Now we will see how these two conditions constrain the error $\U\U^\top -\Uo {\Uo}^\top$.

First we present the following consequence of the RIP assumption~\citep[see][Lemma~2.1]{candes2008restricted}.
\begin{lem}\label{cor:rip1}
Given two $n \times n$ rank-$r$ matrices $\X$ and $\Y$, and a $(2r, \delta)$-RIP measurement operator $\calA$, the following holds:
\begin{equation}\label{eq:rip1}\abs{ \frac{1}{m}\sum_{i=1}^m \ip{\A_i}{\X}\ip{\A_i}{\Y}-\ip{\X}{\Y} }\leq \delta \|\X\|_F \|\Y\|_F.\end{equation}
\end{lem}

 \subsection{First order optimality}

First we will consider the first order condition, $\nabla f(\U) =0 $. For any stationary point $\U$ this implies  
\begin{equation}
\sum_i \ip{\A_i}{\U\U^\top-\Uo{\Uo}^\top}\A_i \U=0.\label{eq:first}
\end{equation}
Now using the isometry property of $\A_i$ gives us the following result.

\begin{lem}\label{lem:first}[First order condition]
For any first order stationary point $\U$ of $f(\U)$, and $\calA$ satisfying the $(2r, \delta)$-RIP~\eqref{eq:rip}, the following holds:
\begin{equation*}
\|(\U\U^\top - \Uo{\Uo}^\top)QQ^\top\|_F \leq \delta \norm{\U\U^\top - \Uo{\Uo}^\top}_F, 
\end{equation*}
where $Q$ is an orthonormal matrix that spans the column space of $\U$.
\end{lem}
This lemma states that any stationary point of $f(\U)$ is close to a global optimum $\Uo$ in the subspace spanned by columns of  $\U$. Notice that the error along the orthogonal direction $\|\Xo Q_{\perp}Q_{\perp}^\top\|_F$ can still be large making the distance between $\X$ and $\Xo$ arbitrarily far.  

\begin{proof}[Proof of Lemma~\ref{lem:first}]
Let $\U=QR$, for some orthonormal $Q$. Consider any matrix of the form $\ZZ Q{R^{-1}}^\top$. The first order optimality condition then implies,
$$\sum_{i=1}^m\ip{\A_i}{\U\U^\top - \Uo{\Uo}^\top}\ip{\A_i}{\U R^{-1}Q^\top \ZZ^\top}=0$$
The above equation together with Restricted Isometry Property (equation~\eqref{eq:rip1}) gives us the following inequality:
$$
\abs{ \ip{\U\U^\top - \Uo{\Uo}^\top}{QQ^\top \ZZ^\top} } \leq \delta \norm{\U\U^\top - \Uo{\Uo}^\top}_F \norm{QQ^\top \ZZ^\top}_F.
$$ 
Note that for any matrix $\A$, $\ip{\A}{QQ^\top \ZZ}=\ip{QQ^\top \A  }{\ZZ}$. Furthermore, for any matrix $\A$, $\sup_{\{\ZZ:\|\ZZ\|_F \leq 1\}} \ip{\A}{\ZZ} =\|\A\|_F$. Hence the above inequality implies the lemma statement.
\end{proof}

\subsection{Second order optimality}

We now consider the second order condition to show that the error along $Q_{\perp} Q_{\perp}^\top$ is indeed bounded well. Let $\nabla^2f(\U)$ be the hessian of the objective function. Note that this is an $n\cdot r \times n \cdot r$ matrix. Fortunately for our result we need to only evaluate the Hessian along the direction $\texttt{vec}(\U -\Uo R)$ for some orthonormal matrix $R$. Here $\texttt{vec}(.)$ denotes writing a matrix in vector form.

\begin{lem}\label{lem:hessian}[Hessian computation]
Let $\U$ be a first order critical point of $f(\U)$. Then for any $r\times r$ orthonormal matrix $R$ and $\Delta_j =\Delta e_j e_j^\top$ ( $\Delta=\U-\Uo R$),
\begin{equation*}
\sum_{j=1}^r \texttt{vec} \left(\Delta_j \right)^\top \left[\nabla^2f(\U)\right] \texttt{vec}\left(\Delta_j \right) \\=   \sum_{i=1}^m (\sum_{j=1}^r4\ip{\A_i}{\U\Delta_j^\top}^2 -2\ip{\A_i}{\U\U^\top -\Uo{\Uo}^\top}^2),
\end{equation*}
\end{lem}

\begin{proof}[Proof of Lemma~\ref{lem:hessian}]
For any matrix $\ZZ$, taking directional second derivative of the function $f(\U)$ with respect to $\ZZ$ we get:
\begin{align*}
&\texttt{vec} \left(\ZZ\right)^\top \left[\nabla^2f(\U)\right] \texttt{vec}\left(\ZZ\right) = \texttt{vec}\left(\ZZ\right)^\top \lim_{t\rightarrow 0} \left[\frac{\nabla f\left(\U+t(\ZZ)\right)-\nabla f(\U)}{t}\right]\\
&= 2\sum_{i=1}^m \bigg[2\ip{\A_i}{\U\ZZ^\top}^2 + \ip{\A_i}{\U\U^\top - \Uo{\Uo}^\top}\ip{\A_i}{\ZZ\ZZ^\top} \bigg]
\end{align*}
Setting $\ZZ=\Delta_j=(\U-\Uo R)e_j e_j^\top$ and using the first order optimality condition on $\U$, we get,
\begin{align*}
&\texttt{vec} \left((\U-\Uo R)e_j e_j^\top \right)^\top \left[\nabla^2f(\U)\right] \texttt{vec}\left((\U-\Uo R)e_j e_j^\top\right)\\
&=  \sum_{i=1}^m 4\ip{\A_i}{\U \Delta_j^\top}^2 +2\ip{\A_i}{\U{\U}^\top -\Uo {\Uo}^\top}\ip{\A_i}{\Delta_j \Delta_j^\top}  \\
&\stackrel{(i)}=  \sum_{i=1}^m 4\ip{\A_i}{\U e_j e_j^\top \Delta_j^\top }^2 +2\ip{\A_i}{\U\U^\top-\Uo {\Uo}^\top}\ip{\A_i}{\Uo e_j e_j^\top (\Uo e_j e_j^\top)^\top}\\
&\stackrel{(ii)}{=}\sum_{i=1}^m 4\ip{\A_i}{\U e_j e_j^\top \Delta_j^\top}^2 -2\ip{\A_i}{\U\U^\top-\Uo {\Uo}^\top}\ip{\A_i}{\U e_j e_j^\top \U^\top - \Uo e_j e_j^\top {\Uo}^\top}. \label{eq:hessian}
\end{align*}
$(i)$ and $(ii)$ follow from the first order optimality condition~\eqref{eq:first}, $$ \sum_{i=1}^m \ip{\A_i}{\U\U^\top}{\U e_j e_j^\top} = \sum_{i=1}^m \ip{\A_i}{\Uo {\Uo}^\top}{\U e_j e_j^\top},$$ for $j =1 \cdots r$. Finally taking sum over $j$ from $1$ to $r$ gives the result.
\end{proof}
Hence from second order optimality of $\U$ we get,
\begin{cor}\label{cor:second}[Second order optimality]
Let $\U$ be a local minimum of $f(\U)$ . For any  $r \times r$ orthonormal matrix $R$,
\begin{equation} \label{eq:second}
\sum_{j=1}^r  \sum_{i=1}^m 4\ip{\A_i}{\U\Delta_j^\top}^2 \geq \frac{1}{2} \sum_{i=1}^m \ip{\A_i}{\U\U^\top -\Uo{\Uo}^\top}^2,
\end{equation}
Further for $\calA $ satisfying $(2r, \delta)$ -RIP (equation~\eqref{eq:rip}) we have,
\begin{equation} \label{eq:second1}
 \sum_{j=1}^r  \|\U e_j e_j^\top(\U -\Uo R)^\top\|_F^2 \geq \frac{1-\delta}{2(1+\delta)}\|\U\U^\top -\Uo{\Uo}^\top\|_F^2.
\end{equation}
\end{cor}

The proof of this result follows simply by applying Lemma~\ref{lem:hessian}. The above Lemma gives a bound on the distance in the factor $(\U)$ space $\|\U(\U -\Uo R)^\top\|_F^2$. To be able to compare the second order condition to the first order condition we need a relation between  $\|\U(\U -\Uo R)^\top\|_F^2$ and $\|\X-\Xo\|_F^2$. Towards this we show the following result.

\begin{lem}\label{lem:combine}
Let $\U$ and $\Uo$ be two $n \times r$ matrices, and $Q$ is an orthonormal matrix that spans the column space of $\U$. Then there exists an $r \times r$ orthonormal matrix $R$ such that for any first order stationary point $\U$ of $f(\U)$, the following holds:
$$\sum_{j=1}^r  \|\U e_j e_j^\top(\U -\Uo R)^\top\|_F^2  \leq  \frac{1}{8}\| \U\U^\top -\Uo{\Uo}^\top\|_F^2 + \frac{34}{8} \| (\U\U^\top -\Uo{\Uo}^\top) QQ^\top\|_F^2.$$
\end{lem}
This Lemma bounds the distance in the factor space ($\|(\U -\Uo R)\U^\top\|_F^2$) with $\| \U\U^\top -\Uo{\Uo}^\top\|_F^2$ and $ \| (\U\U^\top -\Uo{\Uo}^\top) QQ^\top\|_F^2$. Combining this with the result from second order optimality (Corollary~\ref{cor:second}) shows  $\|\U\U^\top -\Uo{\Uo}^\top\|_F^2$  is bounded by a constant factor of $ \| (\U\U^\top -\Uo{\Uo}^\top) QQ^\top\|_F^2$. This implies $\|\Xo Q_{\perp} Q_{\perp}\|_F$ is bounded, opposite to what the first order condition implied (Lemma~\ref{lem:first}). The proof of the above lemma is in Section~\ref{sec:supp}. Hence from the above optimality conditions we get the proof of Theorem~\ref{thm:exact}. 
\begin{proof}[Proof of Theorem~\ref{thm:exact}]
Assuming $\U\U^\top \neq \Uo{\Uo}^\top$, from Lemmas~\ref{lem:first}, ~\ref{lem:combine} and Corollary~\ref{cor:second} we get, 
$$  \left(\frac{1-\delta}{2(1+\delta)} -\frac{1}{8} \right)\|\U\U^\top -\Uo{\Uo}^\top\|_F^2 \leq \frac{34}{8} \delta^2 \norm{(\U\U^\top - \Uo{\Uo}^\top)}_F^2.$$ If $\delta  \leq \frac{1}{5}$ the above inequality holds only if $\U\U^\top = \Uo{\Uo}^\top$.
\end{proof}


\section{Necessity of RIP}\label{sec:necessity}

We showed that there are no spurious local minima only under a
restricted isometry assumption.  A natural question is whether this is
necessary, or whether perhaps the problem \eqref{eq:prob2} never has
any spurious local minima, perhaps similarly to the non-convex PCA
problem $\min_{\U}\norm{\A-\U\U^\top}$.   

A good indication that this is not the case is that \eqref{eq:prob2}
is NP-hard, even in the noiseless case when $\y=\calA(\Xo)$ for
$\rank(\Xo)\leq k$~\cite{recht2010guaranteed} (if we don't require
RIP, we can have each $\A_i$ be non-zero on a single entry in which
case \eqref{eq:prob2} becomes a matrix completion problem, for which
hardness has been shown even under fairly favorable conditions
\cite{hardt2014computational}).  That is, we are unlikely to
have a poly-time algorithm that succeeds for any linear measurement
operator.  Although this doesn't formally preclude the possibility
that there are no spurious local minima, but it just takes a very long
time to find a local minima, this scenario seems somewhat unlikely.

To resolve the question, we present an explicit example of a
measurement operator $\calA$ and $\y=\calA(\Xo)$ (i.e.~$f(\Xo)=0$),
with $\rank(\Xo)=r$, for which \eqref{eq:prob1}, and so also
\eqref{eq:prob2}, have a non-global local minima.

{\it Example 1:} Let \removed{\begin{footnotesize}$$f(\X) = \left(
    \ip{\begin{bmatrix} 1 & 0 \\ 0 & 1 \end{bmatrix}} {
      \begin{bmatrix} X_{11} & X_{12} \\ X_{21} & X_{22}
      \end{bmatrix}} -1 \right)^2 + \left( \ip{\begin{bmatrix} 1 & 0
        \\ 0 & 0 \end{bmatrix}} { \begin{bmatrix} X_{11} & X_{12} \\
        X_{21} & X_{22} \end{bmatrix}} -1 \right)^2.
  $$\end{footnotesize} Then,} $f(\X) =(X_{11} +X_{22} -1 )^2 + (X_{11}
-1)^2$ and consider~\eqref{eq:prob1} with $r=1$ (i.e.~a rank-$1$
constraint).  For \begin{small}$\Xo=\begin{bmatrix} 1 & 0 \\ 0 & 0
  \end{bmatrix}$\end{small} we have $f(\Xo)=0$ and $\rank(\Xo)=1$. But
\begin{small}$\X=\begin{bmatrix} 0 & 0 \\ 0 & 1\end{bmatrix}$
\end{small} is a rank $1$ local minimum with $f(\X) = 1$.

We can be extended the construction to any rank $r$ by simply adding
$\sum_{i=3}^{r+2} (X_{ii} -1)^2$ to the objective, and padding both
the global and local minimum with a diagonal beneath the leading
$2\times 2$ block.

In Example 1, we had a rank-$r$ problem, with a rank-$r$ exact
solution, and a rank-$r$ local minima.  Another question we can ask is
what happens if we allow a larger rank than the rank of the optimal
solution.  That is, if we have $f(\Xo)=0$ with low $\rank(\Xo)$, even
$\rank(\Xo)=1$, but consider \eqref{eq:prob1} or \eqref{eq:prob2} with
a high $r$.  Could we still have non-global local minima?  The answer
is yes...

{\it Example 2:} Let \remove{\begin{footnotesize} \begin{align*} f(\X)
      = \left( \ip{\begin{bmatrix} 1 & 0 & 0 \\0 & 1 & 0 \\ 0 & 0 & 1
          \end{bmatrix}} { \begin{bmatrix} X_{11} & X_{12} &X_{13} \\
            X_{21} & X_{22} & X_{23} \\ X_{31} & X_{32} & X_{33}
          \end{bmatrix}} -1 \right)^2 &+ \left( \ip{\begin{bmatrix} 1
            & 0 & 0 \\0 & 0 & 0 \\ 0 & 0 & 0 \end{bmatrix}} {
          \begin{bmatrix} X_{11} & X_{12} &X_{13} \\ X_{21} & X_{22} &
            X_{23} \\ X_{31} & X_{32} & X_{33} \end{bmatrix}} -1
      \right)^2 \\ &\quad \quad \quad+ \left( \ip{\begin{bmatrix} 0 &
            0 & 0 \\0 & 1 & 0 \\ 0 & 0 & -1 \end{bmatrix}} {
          \begin{bmatrix} X_{11} & X_{12} &X_{13} \\ X_{21} & X_{22} &
            X_{23} \\ X_{31} & X_{32} & X_{33} \end{bmatrix}}
      \right)^2 . \end{align*} \end{footnotesize} Then,} $f(\X)
  =(X_{11} +X_{22} +X_{33} -1 )^2 + (X_{11} -1)^2 + (X_{22}
  -X_{33})^2$ and consider the problem~\eqref{eq:prob1} with a rank
  $r=2$ constraint. We can verify that \begin{footnotesize}$\Xo=\begin{bmatrix} 1 & 0 &
      0\\ 0 & 0 & 0\\ 0 & 0 & 0 \end{bmatrix}$\end{footnotesize} is a
  rank=$1$ global minimum with $f(\Xo)=0$, but 
  \begin{footnotesize}$\X=\begin{bmatrix} 0 & 0 & 0\\ 0 &
      \nicefrac{1}{2} & 0\\ 0 & 0 & \nicefrac{1}{2}
    \end{bmatrix}$\end{footnotesize} is a local minimum with $f(\X) =
  1$.  Also for an arbitrary large rank constraint $r>1$ (taking $r$ to be
  odd for simplicity), extend the objective to $f(\X) =(X_{11} -1)^2 +
  \sum_{i=1}^{(r-1)/2} \left[ (X_{11} +X_{2i, 2i} +X_{(2i+1) , (2i+1)}
    -1 )^2 \right.$ $\left. + (X_{2i, 2i} -X_{(2i+1), (2i+1)})^2 \right]$.  We still
  have a rank-$1$ global minimum $\Xo$ with a single non-zero entry
  $\Xo_{11}=1$, while $\X=(I-\Xo)/2$ is a local minimum with $f(\X)=1$.

\section{Conclusion}
We established that under conditions similar to those required for
convex relaxation recovery guarantees, the non-convex
formulation of matrix sensing~\eqref{eq:prob2} does not exhibit any spurious local
minima (or, in the noisy and approximate settings, at least not
outside some small radius around a global minima), and we can obtain
theoretical guarantees on the success of optimizing it using SGD from
{\em random initialization}.  This matches the methods frequently used
in practice, and can explain their success.  This guarantee is very
different in nature from other recent work on non-convex optimization
for low-rank problems, which relied heavily on initialization to get
close to the global optimum, and on local search just for the final
local convergence to the global optimum.  We believe this is the first
result, together with the parallel work of \citet{ge2016matrix}, on the global
convergence of local search for common rank-constrained problems that
are worst-case hard.

Our result suggests that SVD initialization is not necessary for global convergence, and random initialization would succeed under similar conditions (in fact, our conditions are even weaker than in previous work that used SVD initialization).  To investigate empirically whether SVD initialization is indeed helpful for ensuring global convergence, in Figure \ref{fig2_main} we compare recovery probability of random rank-$k$ matrices for random and SVD initialization---there is no significant difference between the two.

Beyond the implications for matrix sensing, we are hoping these
type of results could be a first step and serve as a model for
understanding local search in deep networks.  Matrix factorization,
such as in \eqref{eq:prob2}, is a depth-two neural network with linear
transfer---an extremely simple network, but already non-convex and
arguably the most complicated network we have a good theoretical
understanding of.  Deep networks are also hard to optimize in the
worst case, but local search seems to do very well in practice.  Our
ultimate goal is to use the study of matrix recovery as a guide in
understating the conditions that enable efficient training of deep
networks.

\subsubsection*{Acknowledgements}
Authors would like to thank Afonso Bandeira for discussions, Jason Lee and Tengyu Ma for sharing and discussing their work.  This research was supported in part by an NSF RI-AF award and by Intel ICRI-CI.

\bibliographystyle{abbrvnat}
{ \bibliography{low_rank_recovery}}
\clearpage
\appendix
\section{Numerical Simulations}\label{sec:simulations}

In this section we present simulation results for performance of gradient descent over $f(\U)$. We consider measurements $y_i =\ip{\A_i}{\Xo}$, where $\A_i$ are i.i.d Gaussian with each entry distributed as $\mathcal{N}(0,\nicefrac{1}{m})$. $\Xo$ is a $100 \times 100$ rank $r$ random p.s.d matrix with $\|\Xo\|_F=1$. $r$ is varied from 1 to 20 in the experiments.

We consider both standard gradient descent and noisy gradient descent~\eqref{eq:grad_updates} with step size $\frac{1}{\|\U\|_2}$. We add noise of magnitude $1e-4$ for the noisy gradient updates. Each method is run until convergence (max of 200 iterations). Let the output of gradient descent be $\widehat{\U}$. A run of this experiment is considered success if the final error $\|\widehat{\U} \widehat{\U}^\top -\Xo\|_F \leq 1e-2$. Each experiment is repeated for 20 times and average probability of success is computed.

We repeat the above procedure starting from both random initialization and SVD initialization. For SVD initialization, the initial point is set to be the rank $r$ approximation of $\sum_{i=1}^m y_i \A_i$ as suggested by~\citet{jain2013low}. In figure~\ref{fig2} we have the plots for the cases discussed above. All of them have phase transition around number of samples $m=2\cdot n\cdot r$. This is in agreement with the results in Section~\ref{sec:main}. $f(\U)$ has no local minima once $m \geq 2\cdot n\cdot r$ and random initialization has same performance as SVD initialization.

In figure~\ref{fig3}, the left two plots show error $\nicefrac{\|\widehat{\U}\widehat{\U}^\top -\Xo\|_F}{\|\Xo\|_F}$ behaves with varying rank and number of samples for random and SVD initializations. The rightmost plot shows the phase transition for rank 10 case for all the methods. Again we notice no significant difference between these methods.

\begin{figure}[!ht]
	\begin{center}
		\includegraphics[width=0.4\textwidth]{figures/exact-init_random_tel.pdf} 
		\hspace{0.1in}
		\includegraphics[width=0.4\textwidth]{figures/exact-init_svd_tel.pdf}\\
	    \includegraphics[width=0.4\textwidth]{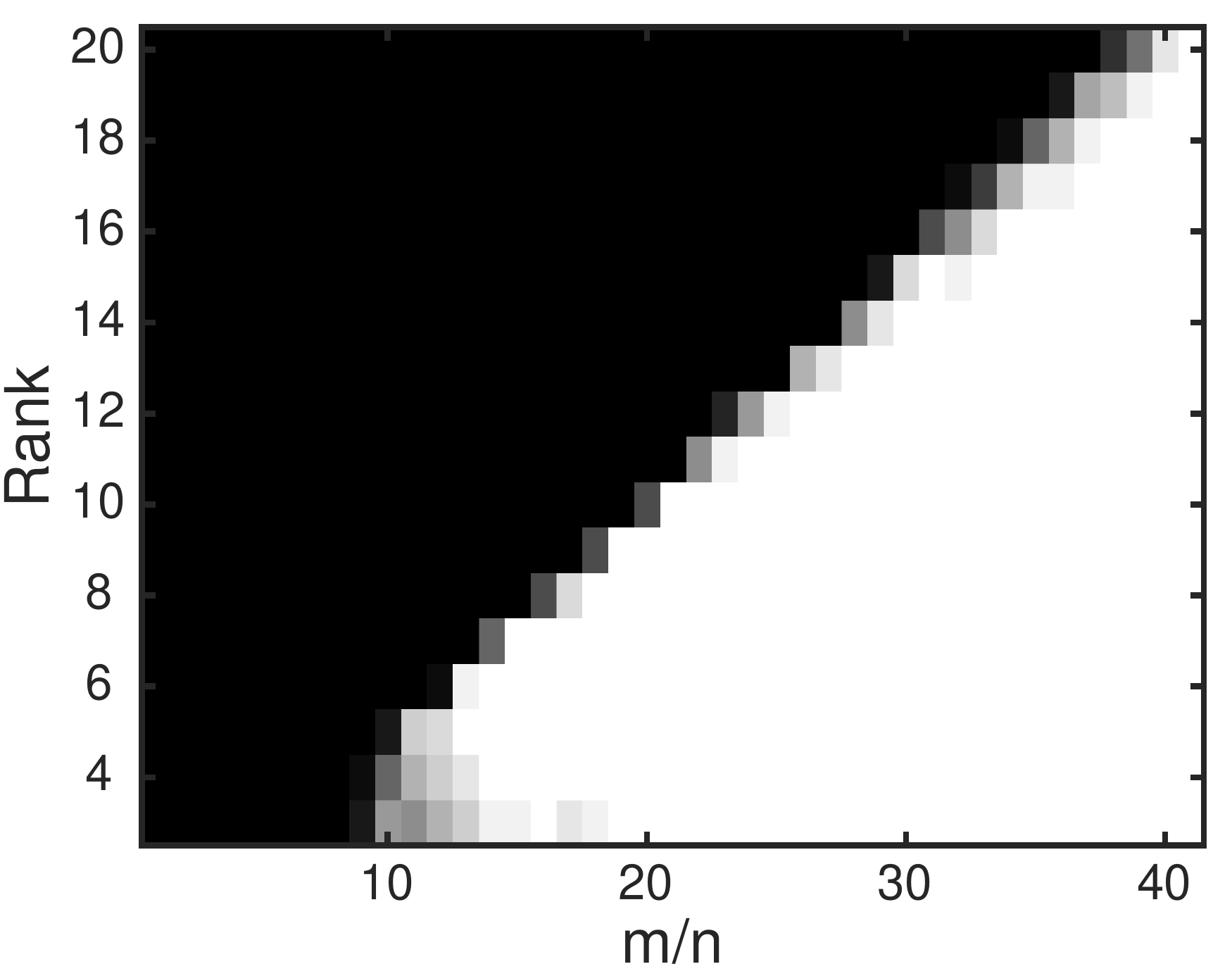} 
	    \hspace{0.1in}
		\includegraphics[width=0.4\textwidth]{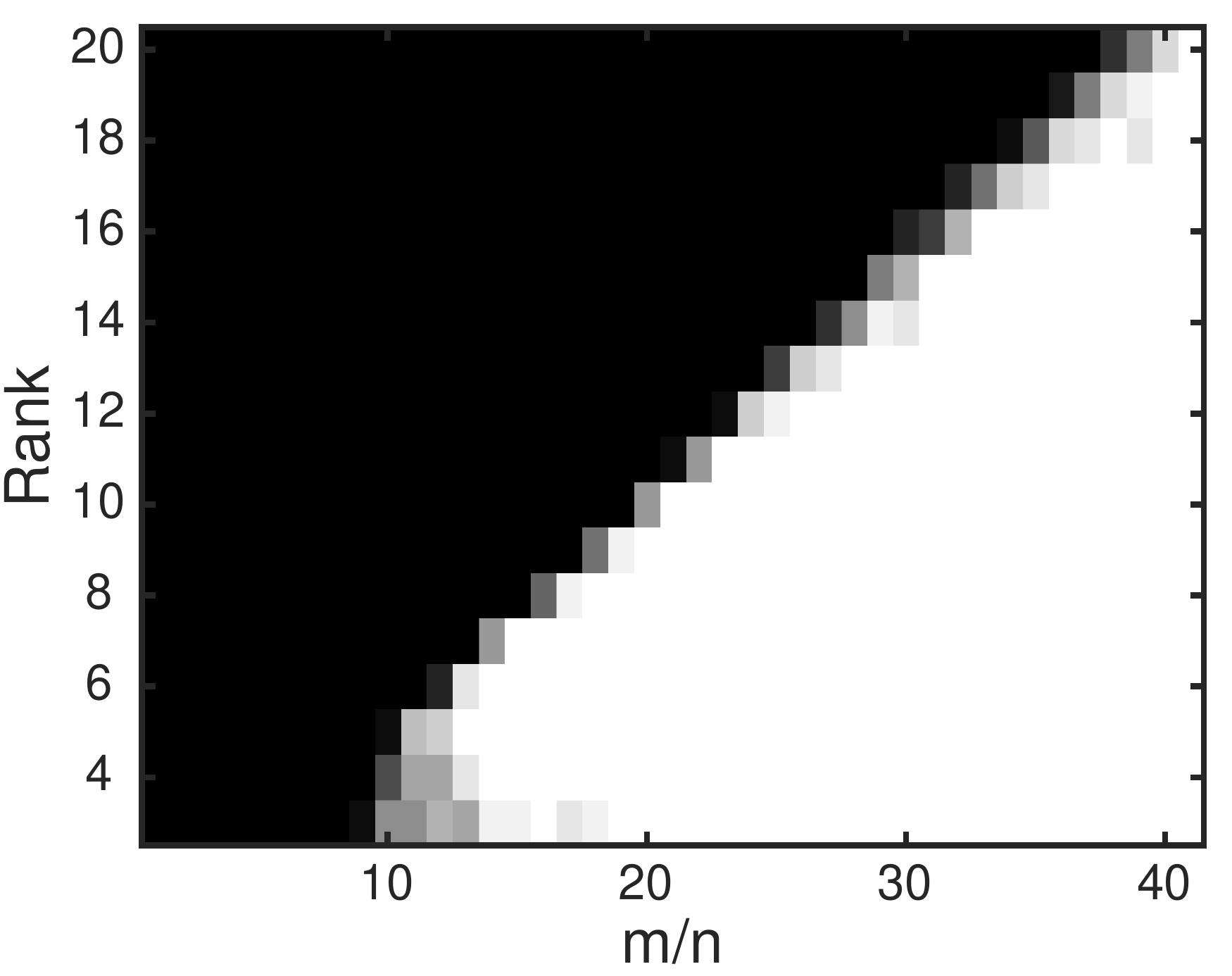}
	\end{center}
	\caption{This figure plots the success probability for increasing number of samples $m$ and various values of rank $r$. The plots on the top are for gradient descent, left for random initialization and the right for SVD initialization. Similarly the bottom plots are for the noisy gradient descent. We notice no significant difference between all these settings. They all have phase transition around $m =2\cdot n\cdot r$.}
{\label{fig2}}
\end{figure}
\vspace{-0.1in}

\begin{figure}[!ht]
	\begin{center}
		\includegraphics[width=0.3\textwidth]{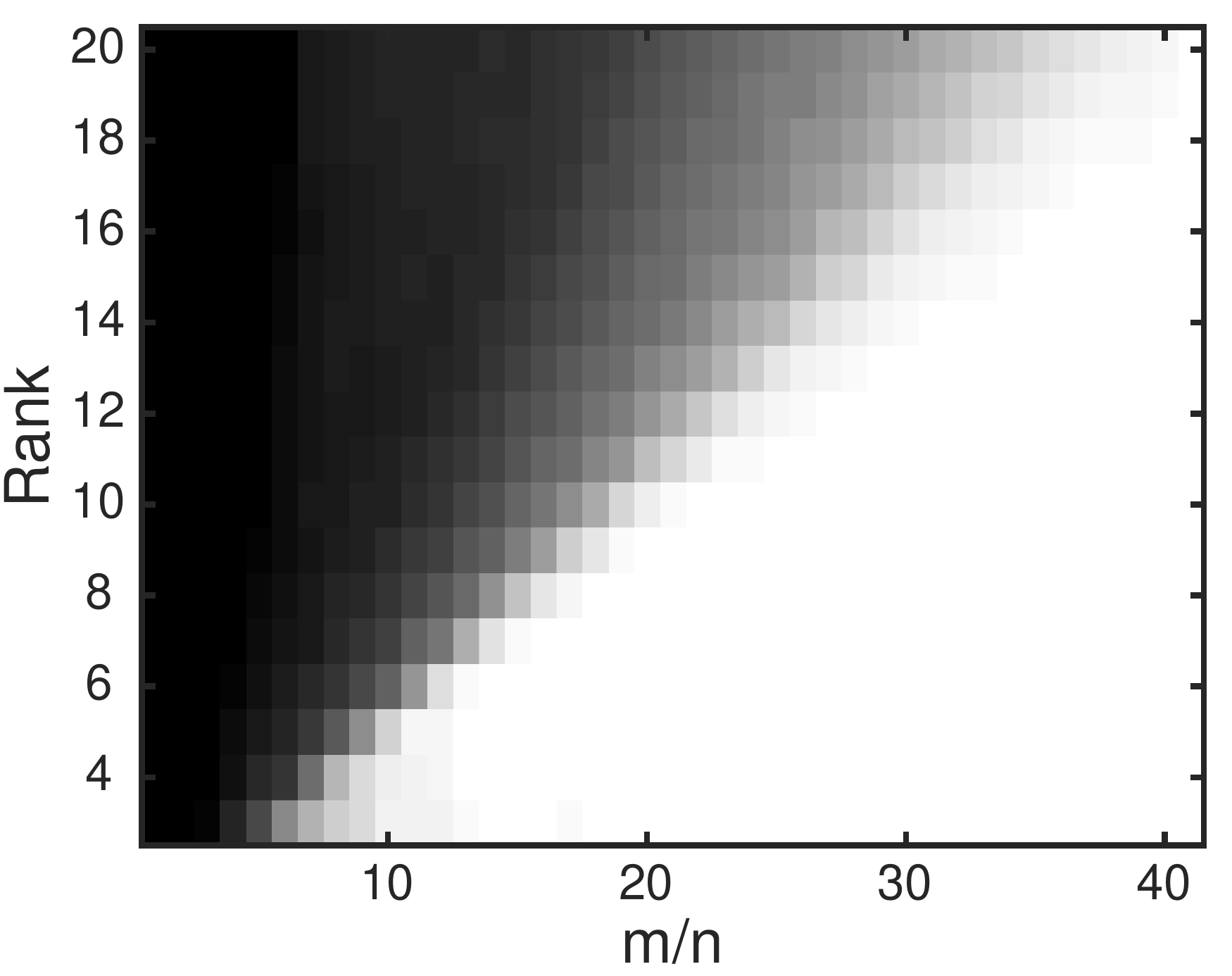} \hspace{0.05in}
		\includegraphics[width=0.3\textwidth]{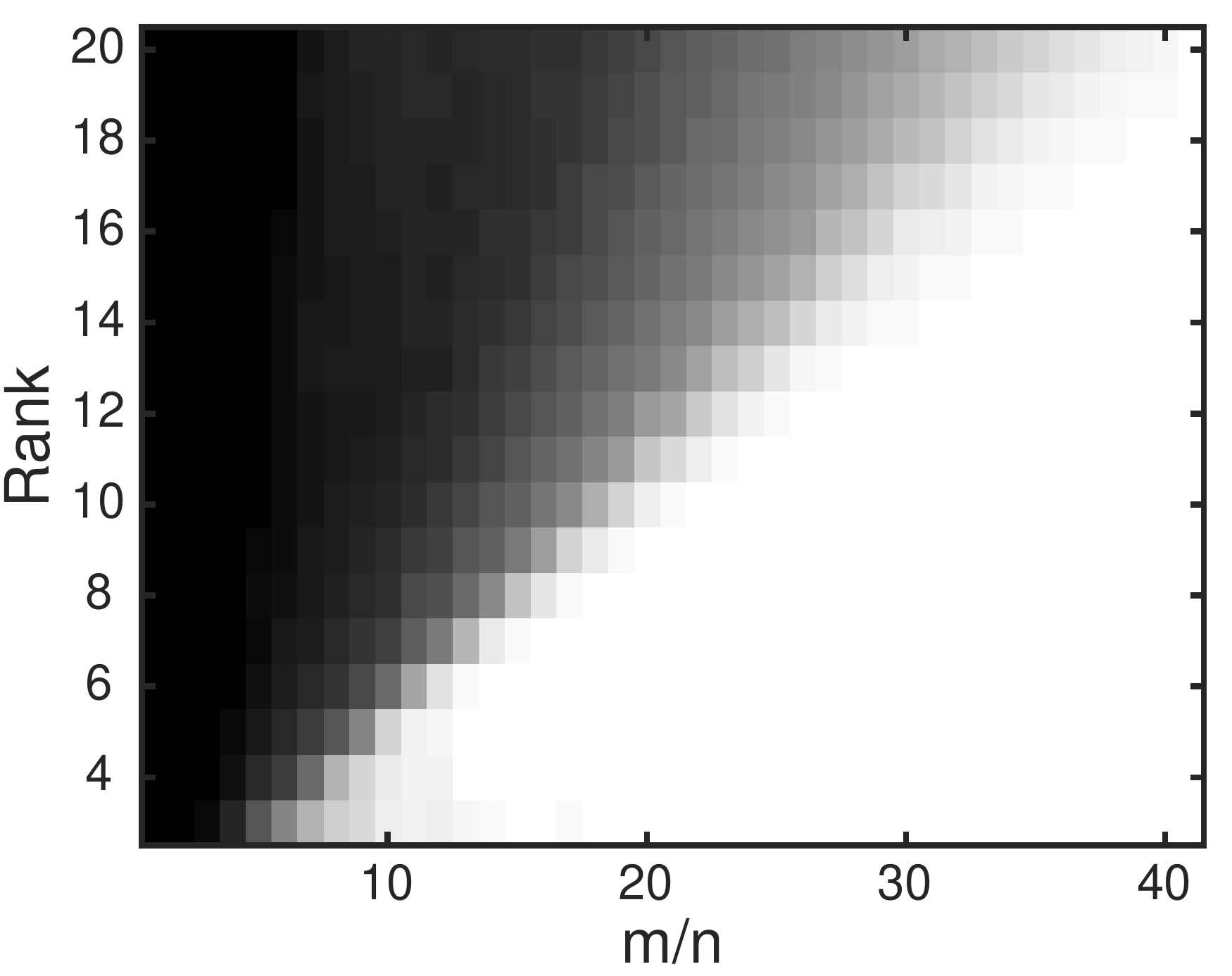}
		\hspace{0.05in}
		\includegraphics[width=0.3\textwidth]{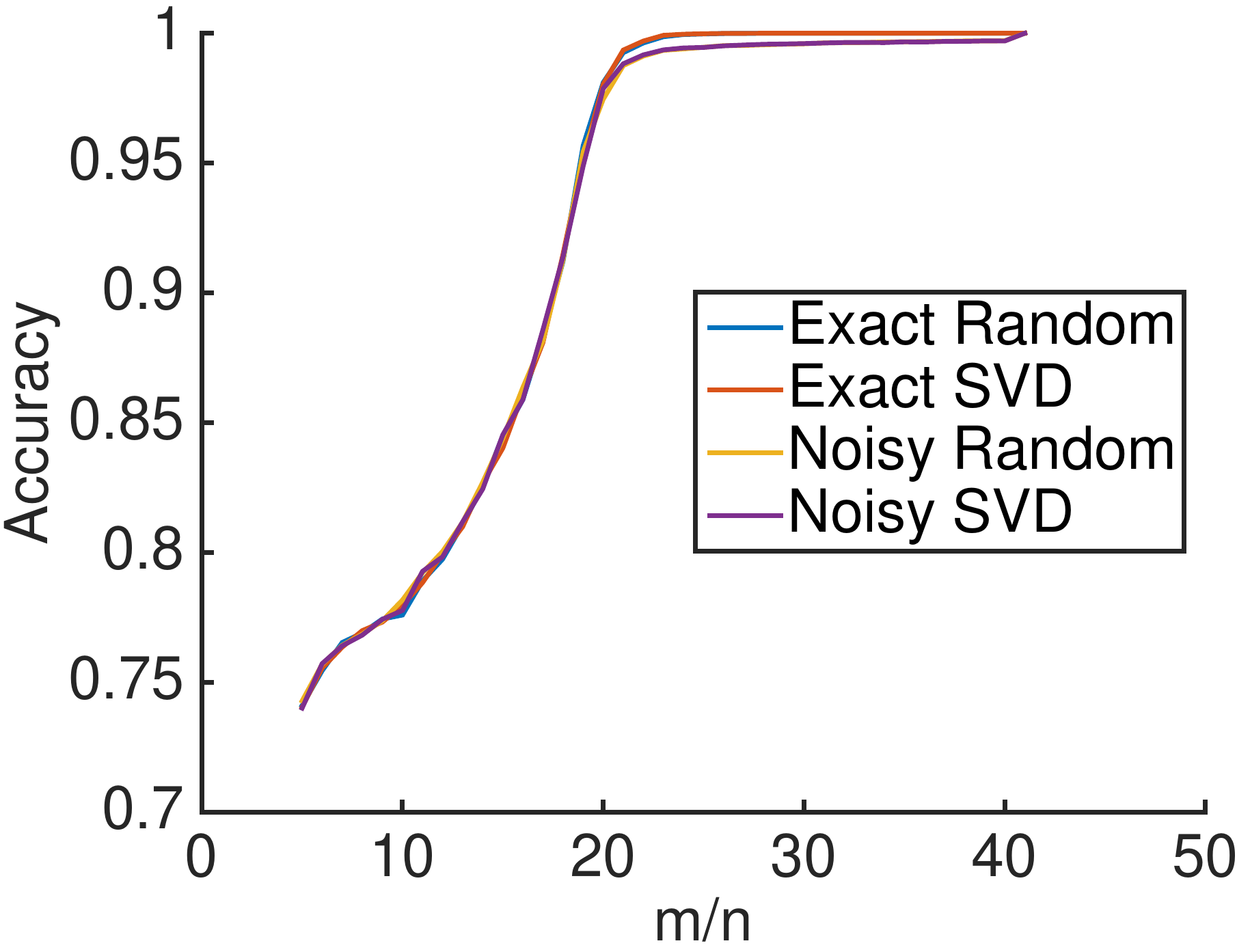}
	\end{center}
	\caption{This figure plots the error $\nicefrac{\|\widehat{U} \widehat{\U}^\top -\Xo\|_F}{\|\Xo\|_F}$ for increasing number of samples $m$. The left plot is for gradient descent with random initialization, center plot corresponds to SVD initialization. Again we notice no difference in error for these two settings. The rightmost figure shows phase transition of low rank recovery for all the different settings when $\Xo$ is rank 10.}
{\label{fig3}}
\end{figure}
\vspace{-0.1in}

\section{Proof for the Noisy Case}\label{sec:proof__noise}

In this section we present the proof characterizing the local minima of problem~\eqref{eq:prob2}. Recall $\vec{y} =\calA(\Xo) + \vec{w}$, where $\Xo$ is a rank-$r$ matrix and $\vec{w}$  is i.i.d.~$\calN(0,\sigma_w^2).$

We consider local optimum that satisfies first and second order optimality conditions of problem~\eqref{eq:prob2}. In particular $\U$ satisfies $\nabla f(\U) =0 $ and $z^\top \nabla^2 f(\U) z \geq 0$ for any $z \in \R^{n\cdot r}$. Now we will see how these two conditions constrain the error $\U\U^\top -\Uo {\Uo}^\top$.
 
 \subsection{First order optimality}

First we will consider the first order condition, $\nabla f(\U) =0 $. For any stationary point $\U$ this implies  
\begin{equation}\label{eq:first_noise}
\sum_i \ip{\A_i}{\U\U^\top-\Uo{\Uo}^\top}\A_i \U=\sum_{i=1}^m w_i \A_i \U.
\end{equation}
Now using the isometry property of $\A_i$ gives us the following result.

\begin{lem}\label{lem:first_noise}[First order condition]
For any first order stationary point $\U$ of $f(\U)$, and $\calA$ satisfying the $(2r, \delta)$-RIP~\eqref{eq:rip}, the following holds:
\begin{equation*}
\|(\U\U^\top - \Uo{\Uo}^\top)QQ^\top\|_F \leq \delta \norm{\U\U^\top - \Uo{\Uo}^\top}_F + 2\sqrt{\frac{(1+\delta) \log(n) }{m}} \sigma_w, 
\end{equation*}
w.p. $\geq 1-\frac{1}{n^2}$, where $Q$ is an orthonormal matrix that spans the column space of $\U$.
\end{lem}
This lemma states that any stationary point of $f(\U)$ is close to a global optimum $\Uo$ in the subspace spanned by columns of  $\U$. Notice that the error along the orthogonal direction $\|\Xo Q_{\perp}Q_{\perp}^\top\|_F$ can still be large making the distance between $\X$ and $\Xo$ arbitrarily big.  

\begin{proof}[Proof of Lemma~\ref{lem:first_noise}]
Let $\U=QR$, for some orthonormal $Q$. Consider any matrix of the form $\ZZ Q{R^{-1}}^\top$. The first order optimality condition then implies,
$$\sum_{i=1}^m\ip{\A_i}{\U\U^\top - \Uo{\Uo}^\top}\ip{\A_i}{\U R^{-1}Q^\top \ZZ^\top}=\sum_{i=1}^m w_i \A_i \U R^{-1}Q^\top \ZZ^\top .$$
The above equation together with Restricted Isometry Property(equation~\eqref{eq:rip1}) gives us the following inequality:
$$
\abs{ \ip{\U\U^\top - \Uo{\Uo}^\top}{QQ^\top \ZZ^\top} } \leq \delta \norm{\U\U^\top - \Uo{\Uo}^\top}_F \norm{QQ^\top \ZZ^\top}_F +2 \sqrt{\frac{(1+\delta) \log(n) }{m}} \sigma_w \|\ZZ^\top\|_F,
$$ 
by Cauchy Schwarz inequality and Lemma~\ref{lem:gauss}. Note that for any matrix $A$, $\ip{\A}{QQ^\top \ZZ}=\ip{\A QQ^\top }{\ZZ}$. Furthermore, for any matrix $A$, $\sup_{\{\ZZ:\|\ZZ\|_F \leq 1\}} \ip{\A}{\ZZ} =\|\A\|_F$. Hence the above inequality implies the lemma statement.
\end{proof}

\subsection{Second order optimality}

We will now consider the second order condition to show that the error along $Q_{\perp} Q_{\perp}^\top$ is indeed bounded well. Let $\nabla^2f(\U)$ be the hessian of the objective function. Note that this is an $n\cdot r \times n \cdot r$ matrix. Fortunately for our result we need to only evaluate the Hessian along the direction $\texttt{vec}(U -\Uo R)$ for some orthonormal matrix $R$. 

\begin{lem}\label{lem:hessian_noise}[Hessian computation]
Let $\U$ be a first order critical point of $f(\U)$. Then for any $r\times r$ orthonormal matrix $R$ and $\Delta=\U-\Uo R$,
\begin{multline*}
\sum_{j=1}^r \texttt{vec} \left(\Delta_j \right)^\top \left[\nabla^2f(\U)\right] \texttt{vec}\left(\Delta_j \right) \\=  \sum_{i=1}^m \left( \sum_{j=1}^r 4\ip{\A_i}{\U\Delta_j^\top}^2 -2\ip{\A_i}{\U\U^\top -\Uo{\Uo}^\top}^2 - 2 w_i \ip{\A_i}{\X -\Xo} \right),
\end{multline*}
\end{lem}

\begin{proof}[Proof of Lemma~\ref{lem:hessian_noise}]
For any matrix $\ZZ$, taking directional second derivative of the function $f(\U)$ with respect to $\ZZ$ we get:
\begin{align*}
&\texttt{vec} \left(\ZZ\right)^\top \left[\nabla^2f(\U)\right] \texttt{vec}\left(\ZZ\right) = \texttt{vec}\left(\ZZ\right)^\top \lim_{t\rightarrow 0} \left[\frac{\nabla f\left(\U+t(\ZZ)\right)-\nabla f(\U)}{t}\right]\\
&= 2\sum_{i=1}^m \bigg[2\ip{\A_i}{\U\ZZ^\top}^2  + \left(\ip{\A_i}{\U\U^\top - \Uo{\Uo}^\top} -w_i\right)\ip{\A_i}{\ZZ\ZZ^\top} \bigg]
\end{align*}

Setting $\ZZ=\Delta_j=(\U-\Uo R)e_j e_j^\top$ and using the first order optimality condition on $\U$, we get,
\begin{align*}
&\texttt{vec} \left((\U-\Uo R)e_j e_j^\top \right)^\top \left[\nabla^2f(\U)\right] \texttt{vec}\left((\U-\Uo R)e_j e_j^\top\right)\\
&=  \sum_{i=1}^m 4\ip{\A_i}{\U \Delta_j^\top}^2 +2(\ip{\A_i}{\U{\U}^\top -\Uo {\Uo}^\top} -w_i )\ip{\A_i}{\Delta_j \Delta_j^\top}  \\
&=  \sum_{i=1}^m 4\ip{\A_i}{\U e_j e_j^\top \Delta_j^\top }^2 +2(\ip{\A_i}{\U\U^\top-\Uo {\Uo}^\top} -w_i)\ip{\A_i}{\Uo e_j e_j^\top (\Uo e_j e_j^\top)^\top}\\
&=\sum_{i=1}^m 4\ip{\A_i}{\U e_j e_j^\top \Delta_j^\top}^2 -2\ip{\A_i}{\U\U^\top-\Uo {\Uo}^\top}\ip{\A_i}{\U e_j e_j^\top \U^\top - \Uo e_j e_j^\top {\Uo}^\top} \\&\qquad \qquad -2w_i\ip{\A_i}{\U e_j e_j^\top \U^\top - \Uo e_j e_j^\top {\Uo}^\top}.\numberthis \label{eq:hessian_noise}
\end{align*}

where the last equality is again by the first order optimality condition~\eqref{eq:first_noise}.
\end{proof}
Hence from second order optimality of $U$ we get,
\begin{cor}\label{cor:second_noise}[Second order optimality]
Let $\U$ be a local minimum of $f(\U)$ . For any  $r \times r$ orthonormal matrix $R$, w.p. $\geq 1-\frac{1}{n^2}$,
\begin{align*}
 \frac{1}{2} \sum_{i=1}^m\ip{\A_i}{\U\U^\top -\Uo{\Uo}^\top}^2 &\leq  \sum_{i=1}^m \sum_{j=1}^r \ip{\A_i}{\U\Delta_j^\top}^2  + \sqrt{ \log(n)} \sigma_w \|\calA(\X -\Xo)\|_2  \\
&\leq   \sum_{i=1}^m \sum_{j=1}^r \ip{\A_i}{\U\Delta_j^\top}^2  + 5\log(n) \sigma_w^2 + \frac{1}{20} \sum_{i=1}^m \ip{\A_i}{\X -\Xo}^2   \label{eq:second_noise}
\end{align*}
Further for $\calA$ satisfying $(2r, \delta)$ -RIP (equation~\eqref{eq:rip}) we have,
\begin{equation} \label{eq:second1_noise}
\frac{1-\delta}{2(1+\delta)}\|\U\U^\top -\Uo{\Uo}^\top\|_F^2 \leq \sum_{j=1}^r \|\U\Delta_j^\top\|_F^2  + \frac{1}{20}\|\X -\Xo\|_F^2+\frac{5\log(n)}{m(1+\delta)} \sigma_w^2. 
\end{equation}
\end{cor}

Hence from the above optimality conditions we get the proof of Theorem~\ref{thm:exact}. 

\begin{proof}[Proof of Theorem~\ref{thm:exact_noise}]

Assuming $\U\U^\top \neq \Uo{\Uo}^\top$, from Lemma ~\ref{lem:combine} and Corollary~\ref{cor:second_noise} we get, with probability $\geq 1-\frac{2}{n^2}$,
\begin{align*}  &\left(\frac{1-\delta}{2(1+\delta)}  \right) \|\U\U^\top -\Uo{\Uo}^\top\|_F^2 \\ & \quad \leq \frac{1}{8} \| \X -\Xo\|_F^2 + \frac{34}{8} \|\X -\Xo QQ^\top\|_F^2 + \frac{1}{20}\|\X -\Xo\|_F^2+\frac{5\log(n)}{m(1+\delta)} \sigma_w^2 \\
& \quad \stackrel{(i)}{\leq}  \left(\frac{1}{8}+\frac{1}{20} \right)\|\X -\Xo\|_F^2 + \frac{34}{8} \left( 2\delta^2 \|\X -\Xo\|_F^2 + 8 \frac{(1+\delta) \log(n)}{m} \sigma_w^2 \right)+\frac{5\log(n)}{m(1+\delta)} \sigma_w^2.
\end{align*}
$(i)$ follows from Lemma~\ref{lem:first_noise}. The above inequality implies,

$$\left(\frac{1-\delta}{2(1+\delta)}  - \frac{1}{8} - \frac{1}{20}   -\frac{34}{4}\delta^2 \right) \|\U\U^\top -\Uo{\Uo}^\top\|_F^2  \leq 34 \frac{(1+\delta) \log(n)}{m} \sigma_w^2 +\frac{5\log(n)}{m(1+\delta)} \sigma_w^2.$$

 If $\delta  \leq \frac{1}{10}$, the above inequality reduces to $ \|\U\U^\top -\Uo{\Uo}^\top\|_F  \leq c \sqrt{\frac{\log(n)}{m}} \sigma_w$, for some constant $c \leq 17 $, w.p $\geq 1-\frac{2}{n^2}$.
\end{proof}


\section{Proof for the High Rank Case}\label{sec:proof_inexact}
In this section we will present the proof for the inexact case, where $\rank(\Xo) \geq r$. Recall that measurements are $\vec{y}=\calA(\Xo)$. 

Let SVD of $\Xo$ be $Q^* \Sigma^* {Q^*}\top$. With slight abuse of notation we  use $\Xo_{jr+1:(j+1)r}$ to denote the $j$th rank $r$ block $Q^*_{jr+1:(j+1)r} \Sigma^*_{jr+1:(j+1)r} {Q^*_{jr+1:(j+1)r}}^\top$, where $Q^*_{jr+1:(j+1)r}$ denotes the restriction of $Q$ to columns $jr+1$ to $(j+1)r$.



\subsection{First order optimality}

First we will consider the first order condition, $\nabla f(\U) =0 $. For any stationary point $\U$ this implies  
\begin{equation}\label{eq:first_inexact}
\sum_i \ip{\A_i}{\U\U^\top-\Uo{\Uo}^\top}\A_i \U=0.
\end{equation}
Now using the isometry property of $\A_i$ gives us the following result.

\begin{lem}\label{lem:first_inexact}[First order condition]
For any first order stationary point $\U$ of $f(\U)$, and $\{ \A_i \}$ satisfying the $(2r, \delta)$-RIP~\eqref{eq:rip}, the following holds:
\begin{equation*}
\|\X -  QQ^\top \Xor\|_F \leq  \delta \norm{\X - \Xor}_F + \|(\Xo -\Xor)QQ^\top\|_F + \delta \|\Xo -\Xor\|_{*}. 
\end{equation*}
where $Q$ is an orthonormal matrix that spans the column space of $\U$.
\end{lem}
This lemma states that any stationary point of $f(\U)$ is close to a global optimum $\Uo$ in the subspace spanned by columns of  $\U$. Notice that the error along the orthogonal direction $\|\Xo Q_{\perp}Q_{\perp}^\top\|_F$ can still be large making the distance between $\X$ and $\Xo$ arbitrarily big.  

\begin{proof}[Proof of Lemma~\ref{lem:first_inexact}]
Let $\U=QR$, for some orthonormal $Q$. Consider any matrix of the form $\ZZ Q R^{-\top}$. The first order optimality condition then implies,
$$\sum_{i=1}^m\ip{\A_i}{\X - \Xor}\ip{\A_i}{\U R^{-1}Q^\top \ZZ^\top}=\sum_{i=1}^m \ip{\A_i}{\Xo -\Xor}\ip{\A_i}{\U R^{-1}Q^\top \ZZ^\top} .$$

Note that $\X - \Xor$ is atmost rank-$2r$. Hence, the above equation together with Restricted Isometry Property(equation~\eqref{eq:rip1}) gives us the following inequality:
\begin{align*}
 \abs{ \ip{\X-  \Xor}{QQ^\top \ZZ^\top} } - &\delta \norm{\X - \Xor }_F \norm{QQ^\top \ZZ^\top}_F \\ &  \quad \quad \leq \frac{1}{m} \sum_{i=1}^m\ip{\A_i}{ \sum_j \Xo_{jr+1:(j+1)r}}\ip{\A_i}{QQ^\top \ZZ^\top}  \\
&\quad \quad \leq \sum_j \ip{\Xo_{jr+1:(j+1)r}}{QQ^\top \ZZ^\top} + \delta \|\Xo_{jr+1:(j+1)r}\|_F \\
&\quad \quad \leq \|(\Xo -\Xor)QQ^\top\|_F + \delta \|\Xo -\Xor\|_{*}.
\end{align*}
The last inequality follows from $\sum_j \|\Xo_{jr+1:(j+1)r}\|_F \leq \|\Xo -\Xor\|_{*}$. The above inequalities are true for any $\ZZ$. 

Further note that for any matrix $\A$, $\ip{\A}{QQ^\top \ZZ}=\ip{\A QQ^\top }{\ZZ}$. Furthermore, for any matrix $A$, $\sup_{\{\ZZ:\|\ZZ\|_F \leq 1\}}$  $ \ip{\A}{\ZZ} =\|\A\|_F$. Hence the above inequality implies the Lemma.
\end{proof}

\subsection{Second order optimality}

We will now consider the second order condition to show that the error along $Q_{\perp} Q_{\perp}^\top$ is indeed bounded well. Let $\nabla^2f(\U)$ be the hessian of the objective function. Note that this is an $n\cdot r \times n \cdot r$ matrix. Fortunately for our result we need to only evaluate the Hessian along the direction $\texttt{vec}(U -\Uo R)$ for some orthonormal matrix $R$. 

\begin{lem}\label{lem:hessian_inexact}[Hessian computation]
Let $\U$ be a first order critical point of $f(\U)$. Then for any $n\times r$  matrix $\ZZ$,
\begin{equation*}
\texttt{vec} \left(\ZZ \right)^\top \left[\nabla^2f(\U)\right] \texttt{vec}\left(\ZZ\right) \\=  \sum_{i=1}^m 4\ip{\A_i}{\U\ZZ^\top}^2 + 2\ip{\A_i}{\U\U^\top -\Uo{\Uo}^\top} \ip{\A_i}{\ZZ \ZZ^\top},
\end{equation*}

Further let $\U$ be a local minimum of $f(\U)$ and $\calA$ satisfying $(2r, \delta)$ -RIP (equation~\eqref{eq:rip}). Then, 
\begin{equation*} \label{eq:second1_inexact}
(1-3\delta) \|\X - \Xor\|_F^2 \leq 4(1+\delta)\sum_{j=1}^r\|\U \Delta_j^\top\|_F^2 + \|\Xo-\Xor\|_F^2 + \delta  \|\Xo -\Xor\|_{*}^2.
\end{equation*}
\end{lem}

\begin{proof}[Proof of Lemma~\ref{lem:hessian_inexact}]
For any matrix $\ZZ$, taking directional second derivative of the function $f(\U)$ with respect to $\ZZ$ we get:
\begin{align*}
&\texttt{vec} \left(\ZZ\right)^\top \left[\nabla^2f(\U)\right] \texttt{vec}\left(\ZZ\right) = \texttt{vec}\left(\ZZ\right)^\top \lim_{t\rightarrow 0} \left[\frac{\nabla f\left(\U+t(\ZZ)\right)-\nabla f(\U)}{t}\right]\\
&= 2\sum_{i=1}^m \bigg[2\ip{\A_i}{\U\ZZ^\top}^2 + \ip{\A_i}{\U\U^\top - \Uo{\Uo}^\top}\ip{\A_i}{\ZZ\ZZ^\top} \bigg].
\end{align*}

Setting $\ZZ=\Delta_j=(\U-\Uo R)e_j e_j^\top$ we get,
\begin{align*}
&\sum_{j=1}^r \texttt{vec} \left((\U-\Uo R)e_j e_j^\top \right)^\top \left[\nabla^2f(\U)\right] \texttt{vec}\left((\U-\Uo R)e_j e_j^\top\right)\\
&=  \sum_{i=1}^m (\sum_{j=1}^r 4\ip{\A_i}{\U e_j e_j^\top(\U-\Uor R)^\top}^2 + 2\sum_{j=1}^r \ip{\A_i}{\U \U^\top - \Uo{\Uo}^\top}\ip{\A_i}{(\U-\Uor R)e_j e_j^\top (\U-\Uor R)^\top}) \\
&\stackrel{(i)}{=}\sum_{i=1}^m ( \sum_{j=1}^r 4\ip{\A_i}{\U \Delta_j^\top}^2 + 2\ip{\A_i}{\U \U^\top - \Uo{\Uo}^\top}\ip{\A_i}{\Uor R (\Uor R)^\top - \X} ).  
\end{align*}
$(i)$ is by the first order optimality condition~\eqref{eq:first_inexact}.

Hence from second order optimality of $U$ we get,
\begin{align}
\sum_{i=1}^m 4\sum_{j=1}^r \ip{\A_i}{\U \Delta_j^\top}^2 \geq \sum_{i=1}^m 2\ip{\A_i}{\X - \Xo}\ip{\A_i}{\X - \Xor}. \label{eq:hessian_inexact1}
\end{align}

\begin{align*}
\frac{1}{m}\sum_{i=1}^m &\ip{\A_i}{\X - \Xo}\ip{\A_i}{\X - \Xor} = \frac{1}{m} \sum_{i=1}^m \ip{\A_i}{\X - \Xor}^2 + \ip{\A_i}{\Xor - \Xo}\ip{\A_i}{\X - \Xor} \\
&\stackrel{(i)}{\geq} (1-\delta)\|\X - \Xor\|_F^2 -  \frac{1}{m} \sum_{i=1}^m  \left( \sum_{j=1}\ip{\A_i}{\Xo_{jr+1:(j+1)r}}\right) \ip{\A_i}{\X - \Xor}\\
&\stackrel{(ii)}{\geq }(1-\delta)\|\X - \Xor\|_F^2 - \sum_{j=1}\ip{\X -\Xor}{\Xo_{jr+1:(j+1)r}} - \delta \sum_{j=1}\|\X - \Xor\|_F \|\Xo_{jr+1:(j+1)r}\|_F   \\
&=(1-\delta) \|\X - \Xor\|_F^2 - \ip{\X -\Xor}{\Xo-\Xor} - \delta \sum_{j=1}\|\X - \Xor\|_F \|\Xo_{jr+1:(j+1)r}\|_F  \\
&\geq (1-\delta) \|\X - \Xor\|_F^2 - \frac{1}{2}\|\X -\Xor\|_F^2 - \frac{1}{2}\|\Xo-\Xor\|_F^2 - \delta \sum_{j=1}\|\X - \Xor\|_F \|\Xo_{jr+1:(j+1)r}\|_F  \\
&\stackrel{(iii)}{\geq}  (1-\delta)\|\X - \Xor\|_F^2 -\frac{1}{2}\|\X -\Xor\|_F^2 - \frac{1}{2}\|\Xo-\Xor\|_F^2- \delta \frac{1}{2}\left( \|\X -\Xor\|_F^2 + \|\Xo -\Xor\|_{*}^2 \right)\\
&= \frac{1-3\delta}{2}\|\X - \Xor\|_F^2 - \frac{1}{2}\|\Xo-\Xor\|_F^2- \frac{ \delta}{2}  \|\Xo -\Xor\|_{*}^2 .\numberthis \label{eq:hessian_inexact2}
\end{align*}
$(i)$ is from using RIP and splitting $\Xo -\Xor$ into rank-$r$ components $\Xo -\Xor =\sum_{j=1}^{\nicefrac{n}{r} -1}\Xo_{jr+1:(j+1)r}$. $(ii)$ follows from using RIP~\eqref{eq:rip1}. $(iii)$ follows from $\sum_{j}\|\Xo_{jr+1:(j+1)r}\|_F \leq \|\Xo -\Xor\|_{*}$.

The Lemma now follows by combining equations~\eqref{eq:hessian_inexact1},~\eqref{eq:hessian_inexact2} and using RIP~\eqref{eq:rip}.
\end{proof}

Hence from the above optimality conditions we get the proof of Theorem~\ref{thm:inexact_sym}. 
\begin{proof}[Proof of Theorem~\ref{thm:inexact_sym}]

Assuming $\U\U^\top \neq \Uor{\Uor}^\top$, from Lemma~\ref{lem:combine} we know,
\begin{align}
\sum_{j=1}^r\|\U \Delta_j^\top\|_F^2  \leq  \frac{1}{8}\| \U\U^\top -\Uor{\Uor}^\top\|_F^2 + \frac{34}{8} \| (\U\U^\top -\Uor{\Uor}^\top) QQ^\top\|_F^2,\label{eq:hessian_inexact3}
\end{align}
for some orthonormal $R$. Hence combining equations~\eqref{eq:hessian_inexact3},with  Lemma~\ref{lem:hessian_inexact} we get,
\begin{align*}
\frac{1-3\delta}{2} \|\X -\Xor\|_F^2 &\leq  \frac{1}{2}\|\Xo-\Xor\|_F^2 +\frac{\delta}{2} \|\Xo -\Xor\|_{*}^2  \\ &\quad \quad + 2(1+\delta)\left( \frac{1}{8}\| \X -\Xor\|_F^2 + \frac{34}{8} \| (\X -\Xor)QQ^\top\|_F^2 \right).
\end{align*}
This implies,
\begin{align}
\frac{1-7\delta}{4} \|\X -\Xor\|_F^2 &\leq  \frac{1}{2}\|\Xo-\Xor\|_F^2 +\frac{\delta}{2} \|\Xo -\Xor\|_{*}^2  + (1+\delta) \frac{17}{2} \| (\X -\Xor)QQ^\top\|_F^2.\label{eq:hessian_inexact4}
\end{align}
Finally from Lemma~\ref{lem:first_inexact} we know,
\begin{align*}
\|\X - \Xor QQ^\top\|_F^2 &\leq  \left(\delta \norm{\X - \Xor}_F + \|(\Xo -\Xor)QQ^\top\|_F + \delta \|\Xo -\Xor\|_{*}\right)^2 \\
&\leq \frac{11}{10} \|(\Xo -\Xor)QQ^\top\|_F^2 + 22 \delta^2 \norm{\X - \Xor}_F^2 +  22 \delta^2 \|\Xo -\Xor\|_{*}^2. \numberthis \label{eq:hessian_inexact5}
\end{align*}
The last inequality follows from just using $2ab \leq a^2 +b^2$.

Combining equations~\eqref{eq:hessian_inexact4} and~\eqref{eq:hessian_inexact5} gives,
\begin{align*}
\left(\frac{1-7\delta}{4} -\frac{17*22(1+\delta) \delta^2}{2} \right) \|\X -\Xor\|_F^2 &\leq \frac{1}{2}\|\Xo-\Xor\|_F^2 +\left( \frac{\delta}{2} +\frac{17*22 \delta^2}{2} \right) \|\Xo -\Xor\|_{*}^2  \\ &\quad + (1+\delta) \frac{17*11}{20} \| (\Xo -\Xor)QQ^\top\|_F^2
\end{align*}
Substituting $\delta =\frac{1}{100}$ gives,
\begin{align*} \|\X -\Xor\|_F^2 &\leq \frac{5}{2}\|\Xo-\Xor\|_F^2 + 12 \delta\|\Xo -\Xor\|_{*}^2  + 10\| (\Xo -\Xor)QQ^\top\|_F^2. \\
&\leq 13\|\Xo-\Xor\|_F^2 + 12\delta\|\Xo -\Xor\|_{*}^2.
\end{align*}
\end{proof}

\section{Proofs for Section~\ref{sec:main}}\label{sec:supp_algo}
In this section we present the proofs for the strict saddle theorem (Theorem~\ref{thm:strict_saddle}) and the convergence guarantees (Theorem~\ref{thm:grad_convergence}). The proofs use the Lemmas developed in Section~\ref{sec:proof_exact} and the supporting Lemmas from Section~\ref{sec:supp}.

\begin{proof}[Proof of Theorem~\ref{thm:strict_saddle}]
From Lemma~\ref{lem:hessian} we know that \begin{align*}
\sum_{j=1}^r \texttt{vec} \left(\Delta_j \right)^\top & \left[\frac{1}{m} \nabla^2f(\U)\right] \texttt{vec}\left(\Delta_j \right) \\&=   \frac{1}{m} \sum_{i=1}^m (\sum_{j=1}^r 4\ip{\A_i}{\U\Delta_j^\top}^2 -2\ip{\A_i}{\U\U^\top -\Uo{\Uo}^\top}^2 \\
&\leq 4(1+\delta)\sum_{j=1}^r \|\U \Delta_j^\top\|_F^2 -2(1-\delta)\|\U\U^\top -\Uo{\Uo}^\top\|_F^2, \numberthis \label{eq:saddle1}
\end{align*}
where the last inequality follows from the RIP~\eqref{eq:rip}. Now applying Lemma~\ref{lem:combine} in equation~\eqref{eq:saddle1} we get,
 \begin{align*}
& \sum_{j=1}^r \texttt{vec} \left(\Delta_j \right)^\top  \left[\frac{1}{m} \nabla^2f(\U)\right] \texttt{vec}\left(\Delta_j \right) \\ &\leq (1+\delta)\left(\frac{1}{2}\| \U\U^\top -\Uo{\Uo}^\top\|_F^2 + 17 \| (\U\U^\top -\Uo{\Uo}^\top) QQ^\top\|_F^2 \right) -2(1-\delta)\|\U\U^\top -\Uo{\Uo}^\top\|_F^2\\
&= 17(1+\delta) \| (\U\U^\top -\Uo{\Uo}^\top) QQ^\top\|_F^2 -\frac{(3-5\delta)}{2}\|\U\U^\top -\Uo{\Uo}^\top\|_F^2 \\
&\stackrel{(i)}{\leq} \left[ 17(1+\delta)\delta^2-\frac{(3-5\delta)}{2}\right] \norm{\U\U^\top - \Uo{\Uo}^\top}_F^2\\
&\stackrel{(ii)}{\leq} -1 \cdot \norm{\U\U^\top - \Uo{\Uo}^\top}_F^2. \numberthis \label{eq:saddle2}
 \end{align*}
 $(i)$ follows from Lemma~\ref{lem:first}. $(ii)$ follows from $\delta \leq \nicefrac{1}{10}$.
Now notice that from lemma~\ref{lem:supp1}
 \begin{align*}
\|\X-\Xo\|_F^2 &\geq 2(\sqrt{2}-1) \|(\U- \Uo R)(\Uo R)^\top\|_F^2 \\
&\geq  2(\sqrt{2}-1) \sigma_r(\Xo)\|\U- \Uo R\|_F^2.  \numberthis \label{eq:saddle3}
\end{align*}

Finally notice that $\Delta_j =\Delta e_j e_j^\top$ have orthogonal columns. Hence, \begin{align*}\lambda_{\min}\left[\frac{1}{m} \nabla^2 (f(\U,\V))\right] &\leq \frac{1}{\|\U- \Uo R\|_F^2}\sum_{j=1}^r \texttt{vec} \left(\Delta_j \right)^\top  \left[\frac{1}{m} \nabla^2f(\U)\right] \texttt{vec}\left(\Delta_j \right) \\
&\stackrel{(i)}{\leq } \frac{-1}{\|\U- \Uo R\|_F^2}\norm{\U\U^\top - \Uo{\Uo}^\top}_F^2 \\
&\stackrel{(ii)}{\leq}  \frac{-2(\sqrt{2}-1) \sigma_r(\Xo)\|\U- \Uo R\|_F^2}{\|\U- \Uo R\|_F^2}\\
& \leq \frac{-4}{5} \sigma_r(\Xo).
\end{align*}
$(i)$ follows from equation~\eqref{eq:saddle2}. $(ii)$ follows from equation~\eqref{eq:saddle3}.
\end{proof}

\begin{proof}[Proof of Theorem~\ref{thm:grad_convergence}]
To prove this theorem we use Theorem 6 of~\citet{ge2015escaping}. We need to show that $f(\U)$ satisfies, 1) strict saddle property, 2) local strong convexity, 3) $f$ is bounded, smooth and has Lipschitz Hessian.

The boundedness assumption easily follows from assuming  we are optimizing over a bounded domain $b$ such that, $\|\Uo\|_F \leq b$. Note that we can have any reasonable upper bound on the optimum and we can easily estimate this from $\sum_i y_i^2$ which is $\geq (1-\delta) \|\Xo\|_F^2$ for the noiseless case. 

Finally all the calculations below are for scaled version of $f(x)$ by $\frac{1}{m}$. Note that this does not change the number of iterations as both smoothness and strong convexity parameters are scaled by the same constant.

{\it Smoothness constant $\beta$:} Recall that smoothness of $f$ is bounded by maximum eigenvalue of Hessian over the domain. Hence,  $\beta =\max_{\ZZ: \|\ZZ\|_F \leq 1} \ZZ^\top \nabla^2 f(\U) \ZZ$. We have computed this projection of Hessian in Lemma~\ref{lem:hessian_noise}. Hence,
\begin{align*}
 \beta &=2 \max_{\ZZ: \|\ZZ \|_F^2 \leq 1} \sum_{i=1}^m \bigg[2\ip{\A_i}{\U\ZZ^\top}^2  + \ip{\A_i}{\U\U^\top - \Uo{\Uo}^\top}\ip{\A_i}{\ZZ\ZZ^\top} \bigg] \\
 &\stackrel{(i)}{\leq} \max_{\ZZ: \|\ZZ \|_F^2 \leq 1} 2 \left( 2(1+\delta) \| \U  \|_F^2 \|\ZZ\|_F^2 + (1+\delta) \|\X -\Xo\|_F \|\ZZ \ZZ^\top\|_F  \right)\\
 &\leq 4(1+\delta)b^2  + (1+\delta) 2 b \leq 5 b^2 +3 b.
\end{align*}
$(i)$ follows from the RIP.

{\it $\rho$- Lipschitz Hessian:} Now we will compute the Lipschitz constant of Hessian of $f(\U)$. We will first bound the spectral norm of difference of Hessian at two points $\U$, $\V$ in terms of $\|\U -\V\|_F$ along orthogonal direction $\ZZ_i$ and combine them to get bound on $\rho$.. Given two $n \times r$ matrices $\U, \V$,
\begin{align*}
& \ip{\nabla^2 f(\U) -\nabla^2 f(\V)}{\ZZ \ZZ^\top} \\&\leq 2 \max_{\ZZ: \|\ZZ \|_F^2 \leq 1} \sum_{i=1}^m \bigg[2\ip{\A_i}{\U\ZZ^\top}^2 + \ip{\A_i}{\U\U^\top - \Uo{\Uo}^\top}\ip{\A_i}{\ZZ\ZZ^\top} \bigg] \\
&\quad - \sum_{i=1}^m \bigg[2\ip{\A_i}{\V \ZZ^\top}^2 +\ip{\A_i}{\V\V^\top - \Uo{\Uo}^\top}\ip{\A_i}{\ZZ\ZZ^\top} \bigg] \\
&\leq 4(1+\delta) (\|\U \ZZ^\top\|_F^2 - \|\V \ZZ^\top\|_F^2) + 2(1+\delta) \| \U \U^\top \V \V^\top\|_F \|\ZZ \ZZ^\top\|_F \\
&\leq 4(1+\delta) \|Z\|_F^2 (\| \U -\V\|_F^2+ 2 \|\U\|_F\| \U -\V\|_F)  + 2(1+\delta) \| \U \U^\top \V \V^\top\|_F \\
&\leq \|Z\|_F^2 \|\U -\V\|_F \left( 8(1+\delta) b +  4(1+\delta) b  \right) \\
&=\|Z\|_F^2 \|\U -\V\|_F \left( 12(1+\delta) b \right) .\numberthis \label{eq:hess_lip}
\end{align*}
Hence, using the variational characterization of the Frobenius norm, the Hessian Lipschitz constant is bounded by $\max{\{\ZZ_i\}} \sum_i \ip{\nabla^2 f(\U) -\nabla^2 f(\V)}{\ZZ_i \ZZ_i^\top}$, where $\ZZ_i$ are orthogonal with $\sum_i \| \ZZ_i\|_F^2 \leq 1$. Hence from equation~\eqref{eq:hess_lip} we get $\rho =O(b)$.

{\it Strict saddle property:}
So far we have shown regularity properties of $f(\U)$. Now we will discuss the strict saddle property. Theorem~\ref{thm:strict_saddle} shows that $\lambda_{\min}\left[\nabla^2 (f(\U))\right] \leq  \frac{-2}{5} \sigma_r(\Xo).$ To use results of ~\cite{ge2015escaping} we need to show this property over an $\eps$ neighborhood of any saddle point $\U$. For this first recall by smoothness, $\|\nabla f(\U) -\nabla f(\V)\|_F \leq \beta \|\U -\V\|_F.$ Therefore $\nabla f(\V) \leq \eps$, when $\|\U -\V\|_F \leq \frac{\eps}{\beta}$. Further we know the Hessian spectral norm is $ \rho$ Lipschitz from equation~\eqref{eq:hess_lip}. Hence, for any direction $\ZZ$,
\begin{align*}
\ZZ^\top \left( \nabla^2 (f(\V)) - \nabla^2 (f(\U)) \right)\ZZ^\top \leq \rho \|\U -\V\|_F \leq \rho \frac{\eps}{\beta}.
\end{align*}

In particular choosing $\ZZ$ to be the projection direction, $\U -\Uo$ implies from Theorem~\ref{thm:strict_saddle},
$$\ZZ^\top \left(\nabla^2 (f(\V)) \right)\ZZ^\top  \leq \frac{-2}{5} \sigma_r(\Xo) +\rho \frac{\eps}{\beta}.$$
Hence for all $\V$ in the bowl of radius $\eps$ around $\U$, where $\eps \leq \frac{\beta}{5\rho}\sigma_r(\Xo) $,
\begin{equation}
\lambda_{\min} \left[\nabla^2 (f(\V)) \right]  \leq \frac{-1}{5} \sigma_r(\Xo).\label{eq:strict_saddle1}
\end{equation}

{\it Local strong convexity:} Finally we need to show that the function is $\alpha$ strongly convex in a neighborhood $\theta$ around the optimum $\Uo R$, for any orthonormal $R$. This easily follows from existing local convergence results for this problem. For example, Lemma~6.1 of~\citet{bhojanapalli2015dropping} states that, for $\|\U -\Uo R\|_F \leq \frac{\sigma_r(\Xo)}{200 \sigma_1(\Xo)} \sigma_r (\Uo R)$, \begin{equation} \ip{\nabla f(\U)}{\U -\Uo R} \geq \frac{2}{3} \eta \|\nabla f(\U)\|_F^2 + \frac{27}{200} \sigma_r(\Uo R)^2 \|\U -\Uo R\|_F^2 . \end{equation} for $\delta =\frac{1}{10}$ and some step size $\eta \propto \frac{1}{\|\Xo\|_2}$. Hence $f(\U)$ is locally strong convex with $\alpha =\frac{27}{200} \sigma_r(\Uo R)^2$ in the neighborhood of radius $\theta = \frac{\sigma_r(\Xo)}{200 \sigma_1(\Xo)} \sigma_r (\Uo R)$ around the optimum.

Substituting these parameters in the Theorem 6 of~\citet{ge2015escaping} gives the result.
\end{proof}

\section{Supporting Lemmas}\label{sec:supp}
In this section we present the supporting results used in the proofs above. 

\begin{lem*}[\ref{lem:combine}]
Let $\U$ and $\Uo$ be two $n \times r$ matrices, and $Q$ is an orthonormal matrix that spans the column space of $\U$. Then there exists an $r \times r$ orthonormal matrix $R$ such that for any first order stationary point $\U$ of $f(\U)$, the following holds:
$$\sum_{j=1}^r  \|\U e_j e_j^\top(\U -\Uo R)^\top\|_F^2  \leq  \frac{1}{8}\| \U\U^\top -\Uo{\Uo}^\top\|_F^2 + \frac{34}{8} \| (\U\U^\top -\Uo{\Uo}^\top) QQ^\top\|_F^2.$$
\end{lem*}

\newcommand{\Qj}{Q_j}
\begin{proof}[Proof of Lemma~\ref{lem:combine}]

To prove this we will expand terms on the both sides in terms of $\U$ and $\Delta =\U -\Uo R$ and then compare. First notice the following properties of $R$ that minimizes $\|\Uo R -\U\|_F$. Let $LSP^\top$ be the SVD of ${\Uo}^\top \U$. Then, $R = LP^\top$. Hence, $R^\top{\Uo}^\top \U  = PSP^\top =  \U^\top \Uo R$ is a PSD matrix. This implies, $\U^\top \Delta = \U^\top \U -\U^\top \Uo R =\U^\top \U -R^\top{\Uo}^\top \U =\Delta^\top \U$.

Let columns of $U$ be orthogonal, else we can multiply $U$ by an orthonormal matrix and $U R$ will satisfy this. Since $U R$ is also local minimum, and $ \U \U^\top = \U R R^\top \U^\top$, results for $\U R$ will also hold for $\U$. Let $Q$ be the orthonormal matrix that spans the column space of $\U$ and $Q_{\perp}Q_{\perp}^\top =I -QQ^\top$. Similarly let $Q_j$ span $\U e_j e_j^\top$. Note that $Q_j$ are orthonormal since columns of $U$ are orthogonal. Hence,
\begin{align*}
\|(\U -\Uo R)e_j e_j^\top \U^\top\|_F^2 &= \|\U e_j e_j^\top \U^\top- \Qj \Qj^\top\Uo R e_j e_j^\top \U^\top -Q_{j \perp}Q_{j \perp}^\top\Uo R e_j e_j^\top \U^\top\|_F^2\\
&= \|\U e_j e_j^\top \U^\top- \Qj \Qj^\top\Uo R e_j e_j^\top \U^\top\|_F^2 +\|Q_{j \perp}Q_{j \perp}^\top\Uo R e_j e_j^\top \U^\top\|_F^2 \\
&\leq \frac{\|\U e_j e_j^\top \U^\top -\Qj \Qj^\top\Uo R e_j e_j^\top  (\Qj \Qj^\top \Uo R)^\top\|_F^2}{2(\sqrt{2}-1)} + \|Q_{j \perp}Q_{j \perp}^\top\Uo R e_j e_j^\top \U^\top\|_F^2. \numberthis \label{eq:combine1}
\end{align*}
The last inequality follows from Lemma~\ref{lem:supp1} and the fact that $ e_j^\top \U^\top \Uo R e_j  \geq 0, \forall j$ as ${\U}^\top \Uo R$ is PSD. Now we will bound the second term in the above equation. The main idea here is to split this term into error between the subspaces of $\X, \Xo$ and then error between their singular values, since both of them are bounded by distance $\|\X -\Xo QQ^\top\|_F.$ Let $Q^*$ be an orthonormal matrix that spans the column space of $\Xo$. Also let $\X = Q \Sigma_U^2 Q^\top$.
\begin{align*}
&\|Q_{j \perp}Q_{j \perp}^\top\Uo R e_j e_j^\top \U^\top\|_F^2 =\trace(e_j^\top R^\top {\Uo}^\top Q_{j \perp}Q_{j \perp}^\top\Uo R e_j e_j^\top \U^\top \U e_j) \\
&=\trace\left(e_j^\top R^\top {\Uo}^\top Q_{j \perp}Q_{j \perp}^\top\Uo R e_j \left[ e_j^\top \U^\top \U e_j   - e_j^\top R^\top {\Uo}^\top \Qj \Qj^\top \Qj \Qj^\top \Uo R e_j + e_j^\top R^\top {\Uo}^\top \Qj \Qj^\top \Uo R e_j \right]\right) \\
&\stackrel{(i)}{\leq} \frac{1}{8} \underbrace{ (e_j^\top R^\top {\Uo}^\top Q_{j \perp}Q_{j \perp}^\top\Uo R e_j)^2}_{\text{term1}} + 2\underbrace{(e_j^\top \U^\top \U e_j   - e_j^\top R^\top {\Uo}^\top \Qj \Qj^\top \Uo R e_j)^2}_{\text{term2}} \\ &\qquad \qquad \qquad \qquad + \underbrace{( Q_{j \perp}Q_{j \perp}^\top \Uo R e_j e_j^\top (\Qj \Qj^\top \Uo R)^\top )^2}_{\text{term3}}. \numberthis \label{eq:combine2}
\end{align*}
where $(i)$ follows from Cauchy-Schwarz inequality.\\ 

We will use the following inequality through the rest of the proof. So we state it first for any matrix $\mat{T}$.
\begin{align*}\sum_{j=1}^r (e_j^\top \mat{T}^\top \mat{T} e_j)^2 &\leq \sum_{j=1}^r \sum_{k=1}^r (e_j^\top \mat{T}^\top \mat{T} e_k)^2 \\
&=  \sum_{j=1}^r e_j^\top \mat{T}^\top \left[ \sum_{k=1}^r  \mat{T} e_k e_k^\top \mat{T}^\top \right] \mat{T} e_j =  \sum_{j=1}^r e_j^\top \mat{T}^\top  \mat{T} \mat{T}^\top \mat{T} e_j\\ & = \|\mat{T}^\top \mat{T} \|_F^2 = \|\mat{T} \mat{T}^\top \|_F^2.  \numberthis \label{eq:combine3}
\end{align*}

Now we will bound each of the terms in equation  .\\
\noindent \textit{Term 1:}
Let, $\mat{T} = Q_{j \perp}Q_{j \perp}^\top\Uo R$. Then applying inequaltiy from equation~\eqref{eq:combine3} we get, \begin{align*}\sum_{j=1}^r (e_j^\top R^\top {\Uo}^\top Q_{j \perp}Q_{j \perp}^\top\Uo R e_j)^2 &=\sum_{j=1}^r (e_j^\top \mat{T}^\top \mat{T} e_j)^2 \\ &\leq \|\mat{T}^\top \mat{T} \|_F^2 = \|R^\top {\Uo}^\top Q_{\perp}Q_{\perp}^\top\Uo R \|_F^2. \numberthis \label{eq:combine4}
\end{align*}
Further,
\begin{align*}
\|R^\top {\Uo}^\top Q_{\perp}Q_{\perp}^\top\Uo R \|_F^2 &=\trace ( {\Uo}^\top Q_{\perp}Q_{\perp}^\top\Uo {\Uo}^\top Q_{\perp}Q_{\perp}^\top\Uo ) \\
&=\trace (Q_{\perp}Q_{\perp}^\top\Xo Q_{\perp}Q_{\perp}^\top\Xo) \\
&\leq \| Q_{\perp}Q_{\perp}^\top\Xo\|_F^2 \leq \| \X -\Xo\|_F^2. \numberthis \label{eq:combine5}
\end{align*}

\noindent \textit{Term 2:}
\begin{align*}
&(e_j^\top \U^\top \U e_j   - e_j^\top R^\top {\Uo}^\top \Qj \Qj^\top \Uo R e_j)^2 \\ &=  (e_j^\top \U^\top \U e_j )^2 + (e_j^\top R^\top {\Uo}^\top \Qj \Qj^\top \Uo R e_j)^2 - 2 e_j^\top \U^\top \U e_j e_j^\top R^\top {\Uo}^\top \Qj \Qj^\top \Uo R e_j \\
&= \| \U e_j e_j^\top \U^\top\|_F^2 + \| \Qj \Qj^\top \Uo R e_j e_j^\top R^\top {\Uo}^\top \Qj \Qj^\top\|_F^2 -2\trace(e_j^\top \U^\top \U e_j e_j^\top R^\top {\Uo}^\top \Qj \Qj^\top \Uo R e_j)\\
&\stackrel{(i)}{= } \| \U e_j e_j^\top \U^\top\|_F^2 +  \| \Qj \Qj^\top \Uo R e_j e_j^\top R^\top {\Uo}^\top \Qj \Qj^\top\|_F^2 - 2\trace( e_j^\top R^\top {\Uo}^\top \U e_j e_j^\top \U^\top \Uo R e_j) \\
&=   \| \U e_j e_j^\top \U^\top -\Qj \Qj^\top \Uo R e_j e_j^\top R^\top {\Uo}^\top \Qj \Qj^\top\|_F^2. \numberthis \label{eq:combine6}
\end{align*}
$(i)$ follows from $e_j^\top \U^\top \U e_j = \|\U_j\|_F^2$ and $\|\U_j\|_F^2 \Qj \Qj^\top =\U e_j e_j^\top \U^\top$. Now from orthogonality of $\Qj$ we have,
\begin{equation}
\sum_{j=1}^r  \| \U e_j e_j^\top \U^\top -\Qj \Qj^\top \Uo R e_j e_j^\top R^\top {\Uo}^\top \Qj \Qj^\top\|_F^2 \leq \| \U \U^\top - QQ^\top \Uo {\Uo}^\top QQ^\top\|_F^2.\label{eq:combine7}
\end{equation}

\noindent \textit{Term 3:}
Finally we bound the last term in equation~\eqref{eq:combine2} similar to the first term, which gives, $$\sum_{j=1}^r ( Q_{j \perp}Q_{j \perp}^\top \Uo R e_j e_j^\top (\Qj \Qj^\top \Uo R)^\top)^2 \leq \|\U \U^\top -\Uo {\Uo}^\top QQ^\top\|_F^2. $$

Substituting the above equations~\eqref{eq:combine4}, \eqref{eq:combine5}, \eqref{eq:combine6} and \eqref{eq:combine7} in \eqref{eq:combine1} and \eqref{eq:combine2} gives the result.
\end{proof}

The following lemma relates the error $\|(\U-\Y)\U^\top\|_F$ with $\|\U \U^\top -\Y \Y^\top\|_F$ under some conditions on $\U$ and $\Y$. This is a generalization of Lemma 5.4 in~\cite{tu2015low} and the proof follows similarly.

\begin{lem}\label{lem:supp1}
Let $\U$ and $\Y$ be two $n \times r$ matrices. Further let $\U^\top \Y =\Y^\top \U$ be a PSD matrix. Then,\begin{equation*} \|(\U-\Y)\U^\top\|_F^2 \leq \frac{1}{2(\sqrt{2}-1)} \|\U \U^\top -\Y \Y^\top\|_F^2.\end{equation*}
\end{lem}

\begin{proof}
To prove this we will expand terms on the both sides in terms of $\U$ and $\Delta =\U-\Y$ and then compare.

\begin{align*}
&\|(\U \U^\top - \Y \Y^\top\|_F^2 =\|(\U \Delta^\top +\Delta \U^\top - \Delta \Delta^\top)\|_F^2 \\
&=\trace\left( \Delta \U^\top \U \Delta^\top  + \U \Delta^\top \Delta \U^\top + \Delta \Delta^\top \Delta \Delta^\top +2 \Delta \U^\top  \Delta \U^\top -2\Delta \Delta^\top \Delta \U^\top -2 \Delta \Delta^\top \U \Delta^\top \right)\\
&\stackrel{(i)}{=}\trace\left( 2\U^\top \U \Delta^\top \Delta  +  (\Delta^\top \Delta)^2 + 2  (\U^\top  \Delta)^2 -4\Delta^\top \Delta \U^\top\Delta \right)\\
&\stackrel{(ii)}{=}\trace\left( 2\U^\top \U \Delta^\top \Delta  + (\Delta^\top \Delta -\sqrt{2} \U^\top \Delta)^2 -2(2-\sqrt{2})\Delta^\top \Delta \U^\top\Delta\right)\\
&\stackrel{(iii)}{\geq} 2\trace\left( \left[\U^\top \U-(2-\sqrt{2}) \U^\top\Delta\right]  \Delta^\top \Delta \right)\\
&= 2\trace\left( \left[(\sqrt{2}-1)\U^\top \U + (2-\sqrt{2}) \U^\top \Y \right]  \Delta^\top \Delta \right)\\
&\stackrel{(iv)}{\geq}  2\trace\left((\sqrt{2}-1)\U^\top \U \Delta^\top \Delta \right).\end{align*}

$(i)$ follows from the following properties of trace: $\trace(\A\B) =\trace(\B\A)$ and $\trace(\A) =\trace(\A^\top)$. $(ii)$ follows from completing the squares. $(iii)$ follows from $\trace(\A^2) \geq 0$. $(iv)$ follows from the hypothesis of the lemma ($\U^\top \Y$ is PSD) and $\trace(\A\B) \geq 0$ for PSD matrices $\A$ and $\B$.

Finally notice that $ \|(\U-\Y)\U^\top\|_F^2 =  \trace(\U^\top \U \Delta^\top \Delta )$. This completes the proof.
\end{proof}

We recall the standard Gaussian random variable concentration here.
\begin{lem}\label{lem:gauss}
Let $w_i \approx \mathcal{N}(0, \sigma_w)$, then $$\sum_{i=1}^m w_i x_i \leq 2\sqrt{\log(n)} \sigma_w \|\vec{x}\| ,$$ with probability $\geq 1- \frac{1}{n^2}$.
\end{lem}

\begin{proof}
Recall $\mathbb{E}\left[ e^{t w_i} \right] = e^{\sigma_w^2 t^2/2}$. Then by Markov's inequality,  $P(\sum_{i=1}^m w_i x_i \geq c\|\vec{x}\|) \leq  \frac{e^{\sigma_w^2 \|\vec{x} \|^2 t^2/2}}{e^{t c \|\vec{x}\|}} \leq e^{ -\nicefrac{c^2}{2\sigma_w^2}}$, by setting $t =\frac{c}{\sigma_w^2 \|\vec{x}\|}$. Choosing $c = 2\sqrt{\log(n)} \sigma_w$ completes the proof. 

\end{proof}

\end{document}